\documentclass{article}

\usepackage{microtype}
\usepackage{graphicx}
\usepackage{subcaption}
\usepackage{booktabs} 

\usepackage{hyperref}

\usepackage[accepted]{icml2026}

\usepackage{amsmath}
\usepackage{amssymb}
\usepackage{mathtools}
\usepackage{amsthm}

\usepackage[capitalize,noabbrev]{cleveref}

\usepackage{algorithm}
\usepackage{algorithmic}
\usepackage{graphicx} 
\usepackage{adjustbox}
\usepackage{wrapfig,lipsum}
\usepackage{capt-of}
\usepackage{caption}
\usepackage{multirow}
\usepackage{tabularx}
\usepackage{isotope} 

\usepackage{url}            
\usepackage{amsfonts}       
\usepackage{nicefrac}       
\usepackage{xcolor,verbatim}         
\usepackage{makecell}

\theoremstyle{plain}
\newtheorem{theorem}{Theorem}[section]
\newtheorem{proposition}[theorem]{Proposition}

\theoremstyle{definition}

\theoremstyle{remark}

\usepackage[textsize=tiny]{todonotes}

\icmltitlerunning{Constrained Multi-Objective Reinforcement Learning with Max-Min Criterion}

\begin{document}

\twocolumn[
  \icmltitle{Constrained Multi-Objective Reinforcement Learning with Max-Min Criterion}



  \icmlsetsymbol{equal}{*}

\begin{icmlauthorlist}
    \icmlauthor{Giseung Park}{can,equal}
    \icmlauthor{Hyunyoung Nam}{kor,equal}
    \icmlauthor{Woohyeon Byeon}{kor}
    \icmlauthor{Amir Leshem}{isr}
    \icmlauthor{Youngchul Sung}{kor}
    \end{icmlauthorlist}

  \icmlaffiliation{can}{Robotics Institute, University of Toronto, 55 St George St, Toronto, Canada.}

    \icmlaffiliation{kor}{School of Electical Engineering, KAIST, Daejeon 34141, Republic of Korea.}
    \icmlaffiliation{isr}{Faculty of Engineering, Bar-Ilan University, Ramat Gan 52900, Israel}
    
    \icmlcorrespondingauthor{Youngchul Sung}{ycsung@kaist.ac.kr}

  \icmlkeywords{Machine Learning, ICML}

  \vskip 0.3in
]



\printAffiliationsAndNotice{\icmlEqualContribution}

\begin{abstract}
    Multi-Objective Reinforcement Learning (MORL) extends standard RL by optimizing policies with respect to multiple, often conflicting, objectives. While max-min MORL has emerged as an effective approach for promoting fairness, its applicability remains limited, particularly when constraints must be incorporated. In this paper, we propose a MORL framework that integrates the max-min criterion with explicit constraint satisfaction. We establish a theoretical foundation for the proposed framework and validate the resulting algorithm through convergence analysis and experiments in tabular settings. We further demonstrate the practical relevance of our approach in simulated building thermal control, multi-objective locomotion control, and greenhouse-gas-emission-aware traffic management. Across these domains, our method effectively balances fairness and constraint satisfaction in multi-objective decision-making.
\end{abstract}

\section{Introduction}

Reinforcement Learning (RL) is a powerful machine learning framework that enables an agent to derive effective decision-making policies by actively engaging with a given environment. In recent years, Multi-Objective Reinforcement Learning (MORL) has gained significant interest because many real-world control problems inherently involve multiple, often conflicting objectives \citep{roijers13survey,yang19envq,hayes22survey,alegre2023sample,teoh2025on,liu2025efficient,kim2025conflictaverse,kim2025fairdice,li2025cola}. MORL extends standard RL to handle simultaneous optimization of multiple objectives.

The central objective of MORL is to identify a policy that yields an expected cumulative return vector lying on the Pareto frontier of the set of achievable return vectors. A common strategy involves optimizing a scalarized function defined over multiple objective returns \citep{roijers13survey,hayes22survey}. This framework seeks to identify a policy $\pi$ that maximizes a scalarized value $f(J_1(\pi),\ldots,J_K(\pi))$, where each $J_k(\pi)$ represents the expected discounted return for the $k$-th objective among $K (\geq 2)$ objectives, and $f: \mathbb{R}^K \to \mathbb{R}$ is a non-decreasing scalarization function such that $J_k(\pi) \geq J_k(\pi'), 1 \leq k \leq K \Rightarrow f(J(\pi)) \geq f(J(\pi'))$. Thus, $f$ plays a key role in imposing the designer's preference among multiple objectives \citep{siddique20fair,cousins2024on}. 

Although much of the MORL literature employs a linear $f$ (that is, the weighted sum: $\max_\pi \sum_{k=1}^K w_k J_k(\pi)$ with fixed $\{ w_k \}_{k=1}^K$) due to its simplicity, the weighted sum does not always accurately represent the preference of a designer, especially regarding fairness among objectives \citep{hayes22survey,park24maxmin}. For instance, imagine a traffic light system managing an intersection where several roads converge with asymmetric arrival rates. Instead of simply aiming to reduce the total sum waiting time for all vehicles across the roads, the designer could prioritize fairness by minimizing the longest individual waiting time among the roads. This helps reduce localized congestion \citep{raeis2021deep} and avoid severe delays for individual drivers.

As such, fairness-driven objectives frequently arise in real-world scenarios and are addressed beyond the standard weighted sum, such as max-min optimization \citep{zehavi2013weighted} or proportionally fair optimization \citep{khan16pf} in MORL. {Specifically, if the Pareto frontier contains an equalization point at which $J_1(\pi) = J_2(\pi) = \cdots = J_K(\pi)$ is satisfied, the max-min solution attains this point \citep{zehavi2013weighted}. This notion of \textit{max-min fairness} is valuable in many applications \citep{regan10rumdp,zehavi2013weighted,saifullah14dag,wang19dag,chakraborty24maxminrlhf,pmlr-v267-eaton25a}, such as mitigating bottlenecks in cloud and edge resource management systems \citep{saifullah14dag,wang19dag} and analyzing societal fairness across diverse demographic groups \citep{pmlr-v267-eaton25a}.}

However, because the standard max-min MORL framework is formulated as an unconstrained optimization problem, it cannot address many real-world problems that involve constraints which must be satisfied. For example, in resource allocation, a MORL-based scheduler may seek to maximize throughput and fairness across task queues while operating under a strict power consumption constraint \citep{chen2020cooling,jiang2020energy}. Similarly, in traffic management systems, a controller must minimize the maximum total waiting time across lanes while maintaining low greenhouse-gas emissions to ensure regional sustainability. Incorporating constraints into the max-min MORL framework therefore substantially broadens its practical applicability.

Despite this motivation, incorporating constraints into unconstrained max-min MORL is not straightforward. On the one hand, existing max-min MORL algorithms may lack solution exactness, for example by optimizing a lower bound of the original max-min objective \citep{fan23welfare,peng2025multiobjectivereinforcementlearningnonlinear} or by being susceptible to inexact gradient estimates \citep{park24maxmin}. On the other hand, existing constrained single-objective RL methods \citep{liu2019ipointeriorpointpolicyoptimization,ha2020learning,vaswani2022nearoptimal,calvo2023state,muller2024truly} do not account for multi-objective reward settings with $K \geq 2$ and fairness across those objectives. Moreover, their analyses are not directly applicable to our setting due to the non-differentiability of the max-min objective.

In this paper, we propose a MORL framework that maximizes max-min fairness among homogeneous objectives while simultaneously incorporating additional quantities as constraints. To address the non-differentiability and nonlinearity of the max-min objective, we reformulate the problem as a convex optimization over occupancy measures \citep{Puterman_2005} and derive a convex program admitted by the dual problem. This reformulation enables us to handle the max-min criterion and the constraints in a unified manner. We further provide a detailed theoretical foundation, including convergence guarantees for the proposed algorithm.

Our main contributions are as follows:

$\bullet$ We introduce a MORL framework that integrates the max-min criterion into constrained optimization and show that fairness and constraint satisfaction can be addressed in a unified manner.

$\bullet$ We propose an iterative algorithm for constrained max-min MORL and establish its theoretical foundations, accompanied by a formal convergence analysis. We also empirically evaluate its behavior in tabular settings.

$\bullet$ We demonstrate the practical relevance of our method through applications to simulated building thermal control, multi-objective locomotion control, and greenhouse-gas-emission-aware traffic management. Our approach consistently achieves a better balance between max-min fairness and constraint satisfaction than the considered baselines.

\section{Background} \label{sec:background}

A multi-objective Markov decision process (MOMDP) is represented as $\langle \mathcal{S}, \mathcal{A}, T, \mu_0, r, \gamma \rangle$, where  \(\mathcal{S}\) and \(\mathcal{A}\) are the sets of states and actions, respectively,  $T$ represents the transition probability distribution, $\mu_0$ specifies the initial state distribution, and $\gamma \in [0,1)$ is the discount factor.
The reward function  \(r: \mathcal{S} \times \mathcal{A} \rightarrow \mathbb{R}^{K+L}, ~K \geq 1, L \geq 0\) is vector-valued  with its $k$-th  element denoted by $r^{(k)} ~ (1 \leq k \leq K+L)$ such that  $|r^{(k)}| \leq r^{(k)}_{\text{max}}$, where $K+L$ is the total number of objectives. For simplicity, we abbreviate $\forall s \in \mathcal{S}, ~ \forall a \in \mathcal{A}$ and $\forall (s,a) \in \mathcal{S} \times \mathcal{A}$ as $\forall s, ~ \forall a$ and $\forall (s,a)$, respectively. 

At each timestep, the agent selects an action \( a \) in the current state \( s \) according to its (stationary) policy \( \pi: \mathcal{S} \to \mathcal{P}(\mathcal{A}) \), where \( \mathcal{P}(\mathcal{A}) \) represents the set of probability distributions in the action space \( \mathcal{A} \). The occupancy measure is defined as $\rho(s,a) := \sum_{s'} \mu_0(s') \sum_{t=0}^\infty \gamma^t \text{Pr}( s_t = s, a_t = a |s_0 = s', \pi^\rho)$ where $\pi^\rho$ is the corresponding stationary policy induced by $\rho$, expressed as $\pi^\rho(a|s) = \frac{\rho(s,a)}{\sum_{a'}\rho(s,a')}$ \citep{Puterman_2005}. Then, the vector return evaluated by $\pi^\rho$ is given by  
\begin{eqnarray}
    J(\pi^\rho) &:=& [J_1(\pi^\rho),\cdots,J_{K+L}(\pi^\rho)]^\top \in \mathbb{R}^{K+L} \nonumber \\
    &=&\mathbb{E}_{\pi^\rho} \left[ \sum_{t=0}^\infty \gamma^t r_t \right]  =
    \sum_{(s,a)} r(s,a) \rho(s,a).
\end{eqnarray}  

In this paper, we define $\mathbb{R}_{+}^L := \{ u \in \mathbb{R}^L | u_l \geq 0, ~  1 \leq l \leq L \}$ and  $\Delta^K := \{ w \in \mathbb{R}^K | \sum_{k=1}^K w_k = 1; ~ w_k \geq 0, ~ 1 \leq k \leq K \}$, i.e., the $(K-1)$-dimensional simplex.

\section{Constrained Max-Min MORL Framework}

\subsection{Theoretical Foundation}\label{subsec:theory}

We consider the MORL setting where the last $L$  of the total $K+L$ objectives should satisfy certain constraints. For theoretical development in this section, we assume that $\mathcal{S}$ and $\mathcal{A}$ are finite. The problem is formulated as follows:  
\begin{eqnarray} 
    &&\hspace{-2em}\max_{\pi^\rho}  ~f( J_1(\pi^\rho),\cdots,J_{K}(\pi^\rho)) + \beta \sum_{s} \mathcal{H}_\rho(s) \rho(s) \label{eq:firstForm}  \\
&& ~~\mbox{s.t.}~~
    J_{K+l}(\pi^\rho) \geq C^{(l)}, ~~ l =1, \cdots, L \label{eq:original_maxmin_f_const}
\end{eqnarray}
where $\mathcal{H}_\rho(s) := - \sum_a \pi^\rho(a|s) \log \pi^\rho(a|s)$ is the entropy of $\pi^\rho(\cdot|s)$, $\rho(s) := \sum_a \rho(s,a)$ is the stationary state distribution in $\mathcal{S}$, $\beta > 0$ is a balancing coefficient, and $\{ C^{(l)} \}_{l=1}^L$ is a set of threshold values. We assume a mild condition that the set $\{ C^{(l)} \}_{l=1}^L$ is chosen by the designer such that the optimization in \eqref{eq:firstForm} and \eqref{eq:original_maxmin_f_const} is feasible, an assumption commonly made in the constrained MDP literature \citep{tessler2018reward,ha2020learning}.

In this paper, we set $f$ the minimum function, i.e., $f( J_1(\pi^\rho),\cdots,J_{K}(\pi^\rho))=\min_{1 \leq k \leq K} J_k(\pi^\rho)$. We note that the entropy term is required to eliminate the indeterminacy of the max-min solution without any regularization \citep{park24maxmin}. The problem reduces to the unregularized formulation as $\beta \to 0$, with the optimality gap decreasing linearly:
\begin{proposition}\label{prop:suboptimality_gap}
The gap between the optimal max-min value of the unregularized problem and that of the regularized problem in \eqref{eq:firstForm} and \eqref{eq:original_maxmin_f_const} with $f=\min$ is upper bounded by $\tfrac{\beta \log|\mathcal{A}|}{1-\gamma}$. (Proof: See Appendix \ref{append:suboptimality_gap}.)
\end{proposition}

Proposition \ref{prop:suboptimality_gap} shows that the regularized problem is a valid approximation of the unregularized criterion.

Since directly optimizing \eqref{eq:firstForm} and \eqref{eq:original_maxmin_f_const} with $f=\min$ and $J_k(\pi^\rho) = \mathbb{E}_{\pi^\rho}[\sum_{t=0}^\infty \gamma^t r^{(k)}_t]$ is non-trivial due to its non-differentiable and nonlinear structure, we address this challenge using the occupancy measure (i.e., stationary distribution \citep{Puterman_2005}) formulation. The above optimization problem with $f=\min$ can be rewritten as  
\begin{equation} \label{eq:maxmin_obj}
    \max_{\rho \geq 0} \min_{1 \leq k \leq K} \bigg(  \sum_{(s,a)} r^{(k)}(s,a) \rho(s,a) \bigg) + \beta \sum_{s} \mathcal{H}_\rho(s) \rho(s)
\end{equation}
\begin{equation} \label{eq:maxmin_flow}
    \sum_{a'} \rho(s',a') = \mu_0(s') + \gamma \sum_{(s,a)} T(s' | s,a) \rho(s,a) ,~\forall s'
\end{equation}
\begin{equation} \label{eq:maxmin_const_c}
    \sum_{(s,a)} c^{(l)}(s,a) \rho(s,a) \geq C^{(l)}, ~~ l =1, \cdots, L
\end{equation}
where \eqref{eq:maxmin_flow} is the Bellman flow equation for the occupancy measure \citep{Puterman_2005}. Here, we use the notation $c^{(l)}(s,a) :=  r^{(K+l)}(s,a), ~l=1,\cdots,L$ to explicitly represent the dimensions associated with the constraint. For example, these quantities can be negative of some costs. Then the formulation in \eqref{eq:maxmin_obj}, \eqref{eq:maxmin_flow}, and \eqref{eq:maxmin_const_c} constitutes a convex optimization problem.

Now we derive a convex optimization problem admitted by the dual of \eqref{eq:maxmin_obj}, \eqref{eq:maxmin_flow}, and \eqref{eq:maxmin_const_c}, which forms the foundation of our algorithm. Given $(u,w)$, let $v^{*}_{u, w}$ be the fixed point of the operator $\mathcal{T}_{u,w}$: 
\begin{align} \label{eq:equivalent_fixed_point}
    [\mathcal{T}_{u,w} v](s) &= \beta \log \sum_a \exp [ \frac{1}{\beta} \{ \sum_{l=1}^L u_l c^{(l)}(s,a) \nonumber \\
    &+ \sum_{k=1}^K w_k r^{(k)}(s,a) + \gamma \hspace{-0.3em}\sum_{s'}\hspace{-0.2em}T(s'|s,a) v(s') \} ],  \forall s. 
\end{align}

\begin{proposition}\label{prop_1}
    The dual problem of \eqref{eq:maxmin_obj}, \eqref{eq:maxmin_flow}, and \eqref{eq:maxmin_const_c} admits the following convex optimization formulation:
    \begin{equation} \label{eq:equivalent_objective}
        \min_{u \in \mathbb{R}_{+}^L, w \in \Delta^K} \mathcal{L}(u,w) = \sum_s \mu_0(s) v^{*}_{u,w}(s) - \sum_{l=1}^L u_l C^{(l)}.
    \end{equation}
    (Proof: See Appendix \ref{append:eqiv_conv_opt}.) 
\end{proposition}

We emphasize that the optimization problem in \eqref{eq:equivalent_objective} is not the dual \textit{per se} of \eqref{eq:maxmin_obj}, \eqref{eq:maxmin_flow}, and \eqref{eq:maxmin_const_c}, since the dual problem itself does not involve the fixed point $v^{*}_{u, w}$ in the objective. Consequently, establishing the convexity of \eqref{eq:equivalent_objective} requires proving the convexity of $v^{*}_{u, w}$ with respect to $(u,w)$, as explained in Appendix \ref{append:eqiv_conv_opt}. Throughout our analysis, we assume that Slater condition holds \citep{boyd2004convex,lee21optidice}, which encompasses the standard constraint feasibility assumption \citep{tessler2018reward,ha2020learning}. 

Proposition \ref{prop_1} hints that $v^{*}_{u, w}$ can be obtained via  \eqref{eq:equivalent_fixed_point} and the weights $(u,w)$ can be obtained by minimizing the loss $\mathcal{L}(u,w)$ in \eqref{eq:equivalent_objective} by some method. However, solving the optimization problem  \eqref{eq:equivalent_objective} is non-trivial because the fixed point $v^*_{u,w}$ does not have a closed-form expression in terms of $(u,w)$. To address this issue, we derive the key properties of $v^*_{u,w}$. We define a policy $\pi^*_{u,w}$ as
\begin{equation} \label{eq:optimal_policy}
    \pi^*_{u,w}(a | s) := \frac{ \exp ( \frac{1}{\beta} Q^*_{u,w}(s,a) )  }{\sum_{a'} \exp ( \frac{1}{\beta} Q^*_{u,w}(s,a') )   }
\end{equation} 
where
\begin{align} \label{eq:eq10}
    Q^*_{u,w}(s,a) &:= \sum_{l=1}^L u_l c^{(l)}(s,a) + \sum_{k=1}^K w_k r^{(k)}(s,a) \nonumber \\
    &+ \gamma \hspace{-0.3em}\sum_{s'}\hspace{-0.2em}T(s'|s,a) v^{*}_{u,w}(s').
\end{align}

Then, $\pi^*_{u,w}$ is an optimal policy for the entropy-regularized RL \citep{haarnoja2017sql} with a scalar reward $\sum_{l=1}^L u_l c^{(l)}(s,a) + \sum_{k=1}^K w_k r^{(k)}(s,a), \forall (s,a)$. We derive the following theorem characterizing the relationship between $\pi^*_{u,w}$ and the gradient of  $v^{*}_{u, w}$:

\begin{theorem}\label{thm1_differentiable}
     For each $s$, $v^*_{u,w}(s)$ is differentiable in $(u,w) \in \mathbb{R}^{L+K}$, and its gradient $\nabla v^*_{u,w}(s) = [\nabla_{u} v^*_{u,w}(s)^\top, \nabla_{w} v^*_{u,w}(s)^\top]^\top$ has the form of
    \begin{equation} \label{eq:eq20}
        \nabla_{u} v^*_{u, w}(s) = v_c^{\pi^*_{u,w}}(s)~~\mbox{and} ~~ \nabla_{w} v^*_{u, w}(s) = v_r^{\pi^*_{u,w}}(s), 
    \end{equation}
    where  $v_c^{\pi^*_{u,w}}(s) \in \mathbb{R}^L$ and $v_r^{\pi^*_{u,w}}(s) \in \mathbb{R}^K$ are the value functions evaluated with the policy $\pi^*_{u,w}$ for the constrained reward  $\{c^{(l)}\}_{l=1}^L$ and the unconstrained reward $\{r^{(k)}\}_{k=1}^K$, respectively.  (Proof: See Appendix \ref{append:differentiable}.)
\end{theorem}

{The main advantage of Theorem~\ref{thm1_differentiable} lies not only in the simplicity of its gradient expression, but also in the fact that the gradient components corresponding to the unconstrained ($w$) and constrained ($u$) parts are obtained in a unified manner and are directly linked to the value function evaluated under $\pi^*_{u,w}$.} This unified form enables us to apply gradient descent to the optimization problem using the gradient $(\nabla_{v} v^*_{u,w}(s),\nabla_{w} v^*_{u,w}(s))$ combined with value iteration.

First, consider the unconstrained part. From \eqref{eq:equivalent_objective}, $\nabla_w \mathcal{L}(u,w)$  is given by $\mathbb{E}_{s \sim \mu_0}[ v_r^{\pi^*_{u,w}}(s) ] \in \mathbb{R}^K$. Consequently, dimensions with smaller values of $\mathbb{E}_{s \sim \mu_0}[ v_r^{\pi^*_{u,w}}(s) ]$ undergo smaller decreases in $w_k$ under gradient descent, resulting in larger $w_k$ relative to other dimensions after update with projection onto $\Delta^K$. As a result, dimensions with smaller values receive larger updated weights, thereby realizing the max-min principle. The constrained component follows a similar logic to enforce feasibility, enabling the simultaneous achievement of max-min fairness and constraint satisfaction. In addition, our framework encompasses unconstrained max-min MORL ($L=0$ with learned $w$) as well as both constrained and unconstrained RL with fixed $w$ such as max-average RL ($w=[1/K, \cdots, 1/K] \in \Delta^K$).

A natural approach to solving the convex optimization problem in \eqref{eq:equivalent_objective}, using Theorem \ref{thm1_differentiable}, is projected gradient descent, since the variables $(u,w)$ lie in the convex set $\mathbb{R}_+^L \times \Delta^K$. The convergence of projected gradient descent depends on the smoothness of the objective function \citep{boyd2004convex,bubeck2015convex}. In our setting, $\mathcal{L}(u,w)$ satisfies the following smoothness property, which is essential for establishing the convergence guarantee in Section \ref{subsec:model_free_algorithm}:

\begin{theorem}\label{thm4_smooth}
    For each $s, v^*_{u,w}(s)$ is smooth with respect to $\| \cdot \|_2$ in $(u,w) \in \mathbb{R}^{L+K}$.  Furthermore, $\mathcal{L}(u,w)$ is $\alpha$-smooth with respect to $\| \cdot \|_2$ in $(u,w) \in \mathbb{R}^{L+K}$ with $\alpha := \frac{1}{\beta (1-\gamma)} \sum_{m=1}^{L+K} \left( \frac{r^{(m)}_{\text{max}}}{1-\gamma} \right)^2$.  (Proof: See Appendix \ref{append:thm_smooth}.)
\end{theorem}

In other words, $\nabla_{(u,w)} \mathcal{L}(u,w)$ is $\alpha$-Lipschitz continuous in $\| \cdot \|_2$. For the proof of Theorem~\ref{thm4_smooth}, we first show that $\nabla v^*_{u,w}$ is differentiable in $(u,w) \in \mathbb{R}^{L+K}$ (equivalently, that $v^*_{u,w}$ is twice-differentiable) in Proposition~\ref{thm2_twice_differentiable}. Then incorporating the closed-form expression of the Hessian of $v^*_{u,w}(s)$ for each $s$ into the generalized mean value inequality constitutes a key ingredient in the proof.

Let $Q^{\pi^*_{u,w}}(s,a) \in \mathbb{R}^{L+K}$ be the action-value function evaluated with the policy $\pi^*_{u,w}$, and define $B^{\pi^*_{u,w}}(s) := \mathbb{E}_{a \sim \pi^*_{u,w}(\cdot|s)} \bigg[ ( Q^{\pi^*_{u,w}}(s,a) - \mathbb{E}_{a' \sim \pi^*_{u,w}(\cdot|s)} [Q^{\pi^*_{u,w}}(s,a')] ) ( Q^{\pi^*_{u,w}}(s,a) - \mathbb{E}_{a' \sim \pi^*_{u,w}(\cdot|s)} [Q^{\pi^*_{u,w}}(s,a')] )^\top    \bigg] \in \mathbb{R}^{(L+K) \times (L+K)}$.

\begin{proposition}\label{thm2_twice_differentiable}
For each $s, v^*_{u,w}(s)$ is twice-differentiable in  $(u,w) \in \mathbb{R}^{L+K}$. Let $|\mathcal{S}|=p$, and suppose the states are enumerated as $\{s_1, \cdots, s_p\}$.  Then, the $(L+K)\times (L+K)$ Hessian matrix  $H[v^*_{u,w}(s_k)], ~ 1 \leq k \leq p$, has the form of
\begin{equation} \label{eq:hessian}
    H[v^*_{u,w}(s_k)] = \frac{1}{\beta} \sum_{l=1}^p [(I_p - \gamma T^{\pi^*_{u,w}})^{-1}]_{kl} B^{\pi^*_{u,w}}(s_l).
\end{equation}
Here, $I_p$ is the $p\times p$ identity matrix; $T^{\pi^*_{u,w}}$ is a $p\times p$ matrix of which 
 $i$-th row and $j$-th column element is given by 
$[T^{\pi^*_{u,w}}]_{ij} = \mathbb{E}_{a \sim \pi^*_{u,w}(\cdot | s_i)}[T(s_j | s_i, a)] ~ (1 \leq i,j \leq p)$; and $[(I_p - \gamma T^{\pi^*_{u,w}})^{-1}]_{kl}$ denotes the $k$-th row and $l$-th column 
 element  of $(I_p - \gamma T^{\pi^*_{u,w}})^{-1}$.  
 (Proof:  See Appendix \ref{append:twice_differentiable}.)
\end{proposition}

This result reflects the fact that the optimization problem in \eqref{eq:equivalent_objective} is a convex program, as established in Proposition~\ref{prop_1}, owing to the positive semidefiniteness of \eqref{eq:hessian} from $(I_p - \gamma T^{\pi^*_{u,w}})^{-1} = \sum_{i=0}^\infty (\gamma T^{\pi^*_{u,w}})^i$ and the relation $H[\mathcal{L}(u,w)] = \mathbb{E}_{s \sim \mu_0} [H[v^*_{u,w}(s)]]$.

Note that $Q^{\pi^*_{u,w}}(s,a)$ in Proposition \ref{thm2_twice_differentiable}
is different from $Q^*_{u,w}(s,a)$ in \eqref{eq:eq10}. By definition in the entropy-regularized RL, $Q^*_{u,w}(s,a) \in \mathbb{R}$ is the cumulative scalarized return plus the cumulative entropy sum from $\pi^*_{u,w}$. On the other hand, $Q^{\pi^*_{u,w}}(s,a) \in \mathbb{R}^{L+K}$ is a cumulative sum of unconstrained rewards and constrained rewards from $\pi^*_{u,w}$ without the entropy sum. Therefore, $[u;w]^\top Q^{\pi^*_{u,w}}(s,a)$ equals to $Q^*_{u,w}(s,a)$ minus the cumulative entropy sum of  $\pi^*_{u,w}$.

\subsection{Algorithm and Convergence Analysis}\label{subsec:model_free_algorithm}

Based on the foundation built in the previous section, we propose an algorithm for constrained MORL with max-min fairness. Note that we need to jointly update the weights $(u,w)$ and the value function, which approximates $v^{*}_{u, w}$.  We adopt the following update method alternating between update of the value function and the weights $(u,w)$.  

First, given a weight $(u,w)$, we update the value function to realize \eqref{eq:equivalent_fixed_point}. For this, we use an 
 action value function $Q$, which approximates $Q^*_{u,w}$. If we plug \eqref{eq:eq10} into the right-hand side of \eqref{eq:equivalent_fixed_point}, we have $v^{*}_{u, w}(s) = [\mathcal{T}_{u,w} v^{*}_{u, w}](s) =  \beta \log \sum_a \exp \left( \frac{Q^*_{u,w}(s,a)}{\beta} \right)$ for each $s$. Using this form of $v_{u,w}^*(s)$, we implement applying $\mathcal{T}_{u,w}$ as updating the Q-function with the following:
\begin{equation} \label{eq:sql}
\begin{aligned}
Q(s,a) \leftarrow {} & [u; w]^\top [c; r] \\
& + \gamma \sum_{s'} T(s'|s,a)\,
\beta \log \sum_{a'} 
\exp\!\left(\frac{Q(s',a')}{\beta}\right),
\end{aligned}
\end{equation}

for all $(s,a)$. In practice, \eqref{eq:sql} is iterated until the maximum change in the $Q$-function across successive iterations is below a given threshold.

Next, we compute an estimated gradient of $\nabla_{(u,w)}\mathcal{L}(u,w)$ at the current weight $(u,w) = (u^m,w^m)$ where $m=1,2,\cdots$ is the iteration index. We have shown that $\nabla_{u} v^*_{u, w}(s) = v_c^{\pi^*_{u,w}}(s), ~ \nabla_{w} v^*_{u, w}(s) = v_r^{\pi^*_{u,w}}(s)$ for each $s,$ where we denote $v_c^{\pi^*_{u,w}}(s) \in \mathbb{R}^L$, $v_r^{\pi^*_{u,w}}(s) \in \mathbb{R}^K$ as the value functions evaluated with the policy $\pi^*_{u,w}$ for constrained reward $c$ and unconstrained reward $r$, respectively. Note that the policy is extracted from the updated Q-function in \eqref{eq:sql} based on the form of \eqref{eq:optimal_policy}. We then update $(u,w)$ using projected gradient descent:
\begin{equation} \label{eq:proj_grad}
    (u^{m+1},w^{m+1}) = \mathcal{P}_{K,L} [ (u^m,w^m) - l_{\text{w}}  \nabla_{(u,w)}\mathcal{L}(u^m,w^m)  ]
\end{equation}
where $l_{\text{w}}$ is a learning rate for $(u,w)$ and $\mathcal{P}_{K,L}[\cdot]$ is the projection onto the $\mathbb{R}_+^L \times \Delta^K$. We use the convex optimization method from \citet{DBLP:journals/corr/WangC13a} to project onto the simplex $\Delta^K$, and apply non-negativity clipping for projection onto $\mathbb{R}_+^L$. Note that the projection onto $\Delta^K$ is numerically stable as it is fully deterministic and avoids randomized procedures. In addition, its complexity is $O(K \log K)$ \citep{DBLP:journals/corr/WangC13a} which is  lightweight due to the logarithmic term.

We iterate this process for each $m$, and Algorithm \ref{alg:constrained_morl} presents the pseudocode of our algorithm.

\begin{algorithm}[!ht] 
\caption{ Constrained Max-Min MORL Algorithm} \label{alg:constrained_morl}
\begin{algorithmic}[1]
    \STATE $Q^{0} \in \mathbb{R}^{|\mathcal{S}||\mathcal{A}|}$: initialized Q-function, $\text{ITER}$: total iteration number, $l_{\text{w}}$: learning rate for the weights $(u,w)$
    \STATE Initialize weights $u^0 \in \mathbb{R}^L_+$ and $w^0 \in \Delta^K$.
    \FOR{$m = 1, 2, \cdots, \text{ITER}$}
        \STATE{$Q = Q^{m-1}$}
        \WHILE{not terminated}
            \STATE Update $Q$ in \eqref{eq:sql} with $[u;w] = [u^m;w^m]$.
        \ENDWHILE
        \STATE{$Q^{m} = Q$}
        \STATE Compute $\tilde{\nabla}_{(u,w)}\mathcal{L}(u^m,v^m)$, an estimated gradient of $\nabla_{(u,w)}\mathcal{L}(u^m,w^m)$ using $\pi^m( \cdot | s) = \text{softmax}  \{ Q^{m}(s,\cdot) / \beta  \}$ based on \eqref{eq:eq20}.
        \STATE $(u^{m+1},w^{m+1}) = \mathcal{P}_{K,L} [ (u^m,w^m) - l_{\text{w}}  \tilde{\nabla}_{(u,w)}\mathcal{L}(u^m,w^m)  ]$.  
    \ENDFOR
    \STATE \textbf{return} $\pi(\cdot|s) = \text{softmax}  \{ Q^{\text{ITER}}(s,\cdot) / \beta  \}, \forall s$.
\end{algorithmic}
\end{algorithm}

We now provide our convergence analysis under the following assumption. Let $\rho^{\pi^*_{u,w}}(s) = \sum_{t=0}^\infty \gamma^t \text{Pr}(s_t=s | \pi^*_{u,w}, T, \mu_0)$ be the discounted state visitation frequency of policy $\pi^*_{u,w}$ at state $s$ \citep{Puterman_2005}.

\textit{Assumption}
There exists at least one state $s \in \mathcal{S}$ such that $\rho^{\pi^*_{u,w}}(s) > 0$ and the centered action-value vectors $\bigl\{\,Q^{\pi^*_{u,w}}(s,a) - \mathbb{E}_{a' \sim \pi^*_{u,w}(\cdot\mid s)}[\,Q^{\pi^*_{u,w}}(s,a')\,]\;:\;a\in\mathcal{A}\bigr\}$ span $\mathbb{R}^{K+L}$.


This condition fails only in degenerate multi-objective settings when for \textit{every} state $s\in\mathcal{S}$ with $\rho^{\pi^*_{u,w}}(s) > 0$, the set $S_{\text{center}}(s)$ lies entirely within an affine subspace of dimension less than $K+L$. Under \textit{Assumption}, the Hessian $H[\mathcal{L}(u,w)]$ is positive definite. We note that the property $\pi^*_{u,w}(a|s) > 0$ for all $(s,a)$, implied by \eqref{eq:optimal_policy}, plays an essential role in ensuring the positive definiteness of $H[\mathcal{L}(u,w)]$, which constitutes an additional technical benefit of the entropy regularization in \eqref{eq:maxmin_obj}. (See Appendix \ref{append:convergence_algo_assumption} for more details.)

Moreover, under Slater condition with strict feasibility of the $L$ constraints, we can restrict $u \in \mathbb{R}^L_+$ to a compact set $\mathcal{C}$ \cite{muller2024truly}, ensuring the existence of $\lambda := \min_{(u,w) \in \mathcal{C} \times \Delta^K} \lambda_{\text{min}}(H[\mathcal{L}(u,w)])$, where $\lambda_{\min}(\cdot)$ denotes the minimum eigenvalue of a real symmetric matrix. (See Appendix \ref{append:existence_lambda} for more details.) Furthermore, $\lambda$ satisfies $0 < \lambda \leq \alpha$ \citep{bubeck2015convex}.

Theorem \ref{thm_convergence_algo} provides a formal guarantee of convergence for Algorithm \ref{alg:constrained_morl} under approximate Q-updates.

\begin{theorem}\label{thm_convergence_algo}
    Let $(u^*,w^*)$ denote an optimal solution to \eqref{eq:equivalent_objective} and $l_{\text{w}} = 1/\alpha$. For each outer-loop index $m \geq 1$ in Algorithm \ref{alg:constrained_morl}, let $Q^*_{u^m, w^m}$ denote the fixed point of \eqref{eq:sql} with $[u;w] = [u^m;w^m]$, and let $Q^m$ denote the Q-function after completing the $m$-th inner-loop update. For each $m$, assume $\| Q^m - Q^*_{u^m, w^m} \|_\infty < \epsilon$ for some $\epsilon > 0$. Then for $m \geq 1$,
    \begin{align} \label{eq:final_convergence_result}
        \| [u^m;w^m] - [u^*;w^*] \|_{2} &\leq  (1-\frac{\lambda}{\alpha})^m \| [u^0;w^0] - [u^*;w^*] \|_{2} \nonumber \\
        &+ \frac{\sqrt{|\mathcal{S}|}}{\lambda}  \sqrt{\sum_{i=1}^{K+L} \{ r^{(i)}_{\text{max}} \}^2} \frac{1+\gamma }{(1-\gamma)^2} \epsilon.
    \end{align}
(Proof: See Appendix \ref{append:convergence_algo_proof}.)
\end{theorem}

Theorem \ref{thm_convergence_algo} establishes that the error decreases geometrically at rate $O\left((1-\tfrac{\lambda}{\alpha})^m\right)$, up to $O(\epsilon)$.  For completeness, Appendix \ref{append:convergence_algo_proof_degenerate} provides an analysis of the case without \textit{Assumption}, supporting the practical implementation of our algorithm beyond \textit{Assumption} in Section~\ref{subsec:constrained_deep}.

\section{Related Work } 

\textbf{MORL} ~ The dominant approach in MORL is utility-based \citep{roijers13survey,hayes22survey}, where the goal is to find an optimal policy $\pi^* = \arg \max_{\pi} f(J(\pi))$ given a non-decreasing scalarization function $f: \mathbb{R}^K \rightarrow \mathbb{R}$. When $f$ is linear, each non-negative weight vector defines a scalarized MDP \citep{boutilier99optimal}, motivating work on learning a single model capable of generating policies across a continuum of preferences  {\citep{abels19moq,yang19envq,basaklar23pdmorl,hung23qpensieve,lu23capql,park2025reward,li2025cola}. This family of approaches is known as multi-policy MORL \citep{roijers13survey,hayes22survey}.

For nonlinear scalarization functions, however, Bellman optimality no longer holds in its standard form due to the loss of linearity, making optimization substantially more difficult \citep{roijers13survey,hayes22survey}. Algorithms that directly optimize a given nonlinear scalarized objective fall under the single-policy MORL category \citep{roijers13survey,hayes22survey}. Most work in this category considers a welfare function \citep{siddique20fair,cousins2024on,kim2025fairdice} as the nonlinear scalarization $f$. Note that single-policy and multi-policy methods are complementary rather than interchangeable \citep{roijers13survey,hayes22survey}. }

\textbf{Unconstrained Max-min MORL} ~ {Max-min MORL studies the case where $f=\min$, aiming to enforce max-min fairness. This formulation is  useful in many applications such as mitigating bottlenecks in cloud and edge resource management systems \citep{saifullah14dag,wang19dag}. Several  studies optimize proxy objectives related to the unconstrained max-min formulation, for example, maximizing a conservative lower bound of $\mathbb{E}_\pi \left[ \min_{1 \leq k \leq K}( \sum_{t=0}^\infty \gamma^t {r_t^{(k)}} ) \right]$ \citep{fan23welfare,peng2025multiobjectivereinforcementlearningnonlinear}, or maximizing the total return while enforcing per-group performance constraints \citep{pmlr-v267-eaton25a}.} 

\citet{park24maxmin} proposed a tractable approach for exact unconstrained max-min MORL using Gaussian smoothing to estimate gradients. However, this approach requires maintaining multiple network copies, increasing computational overhead. Furthermore, the gradient estimates are inherently inexact, as Gaussian smoothing of a convex function yields a convex upper bound rather than the true function \citep{nesterov2017random}. In contrast, our method produces direct, theoretically grounded gradient estimates and extends naturally to constrained MORL. { Recently, \citet{byeon2025multiobjectivereinforcementlearningmaxmin} introduced an alternative unconstrained max-min MORL formulation based on a game-theoretic framework \citep{daskalakis2018last,NEURIPS2019_Miryoosefi}. However, their method does not address the incorporation of constraints. }

\textbf{Constrained RL} ~ Many approaches to constrained MDPs reformulate the problem with a scalar reward (i.e., a special case of \eqref{eq:firstForm} and \eqref{eq:original_maxmin_f_const} with $K=1$ and without $f$) into an unconstrained one by augmenting the objective with a weighted sum of constraint violations, typically via a Lagrangian formulation \citep{achiam2017constrained, tessler2018reward, paternain2019constrained, ha2020learning, vaswani2022nearoptimal, calvo2023state, muller2024truly}. The motivation for this line of work is that the Lagrangian relaxation exhibits no duality gap, even when the original problem is non-convex with respect to the policy \citep{paternain2019constrained}. However, these approaches do not consider the multi-objective reward setting in \eqref{eq:firstForm} and \eqref{eq:original_maxmin_f_const} with $K \geq 2$. Moreover, applying them directly to our setting is non-trivial, since $f = \min$ introduces non-differentiability in \eqref{eq:firstForm}.

To resolve this, we reformulate our problem as a convex program using occupancy measures and then derive another convex program admitted by the dual problem. In particular, we show that both the max-min criterion and the constraints can be satisfied by jointly updating the weights $u$ and $w$, a simple yet effective approach that to our knowledge has not been explored in the constrained MDP literature. Although \citet{lee2022coptidice} also leverages convex analysis with occupancy measures, its focus is on constrained single-objective RL with a scalar reward (i.e., $K=1$) in an offline setting. Unlike our work, it neither addresses max-min fairness across multiple objectives nor provides a convergence analysis.

Several recent works have incorporated constraints into MORL, but under settings that differ from our framework which explicitly integrates max-min optimization. Some approaches learn preference-vector-conditioned policies to explore trade-offs among objectives \citep{huang21lp3,kim2025conflictaverse}. While complementary to our work, it remains unclear how to select a preference vector such that the resulting policy exactly corresponds to the optimal solution under nonlinear scalarizations, such as the max-min criterion. In contrast, our approach directly solves a single-policy constrained max-min optimization problem, mirroring the distinction between single- and multi-policy methods in unconstrained MORL \citep{roijers13survey,hayes22survey}. \citet{lin24cmorladapt} study offline constrained MORL, where policies are trained on offline data and adapted to target preferences using additional demonstrations. In contrast, our work focuses on online learning and does not rely on offline data or demonstrations. \citet{liu2025efficient} improve Pareto frontier coverage in unconstrained MORL by training multiple policies via constrained optimization, rather than directly addressing constrained MORL.

\section{Experiments }

In this section, we present experimental validations of our theoretical analysis and algorithm. Section \ref{subsec:tabular} examines the convergence behavior of our method in tabular settings. In Section \ref{subsec:constrained_deep}, we further demonstrate the practical relevance of our approach through applications beyond tabular settings including building thermal control, multi-objective locomotion control, and greenhouse-gas-emission-aware traffic management.

\subsection{Tabular Settings} \label{subsec:tabular}

We conducted experiments in tabular settings to evaluate the convergence of our algorithm. Constrained MOMDPs were randomly generated, after a feasibility check, within two widely used classes of structured MDPs. (See Appendix \ref{append:tabular} for details on the experimental setup, feasibility check, and baselines.) First, bipartite state graphs partition the state space into two disjoint subsets, enforcing transitions between them at alternating time steps. This structure captures temporal dynamics in systems with role alternation or interleaving phases \citep{littman1994markov}. Second, hierarchical MDPs organize the state space into multiple levels or stages, where transitions flow sequentially from one level to the next. This reflects tasks with subgoals or temporal abstraction \citep{dietterich2000hierarchical}.

The optimal value for each constrained MOMDP was computed by solving \eqref{eq:maxmin_obj}, \eqref{eq:maxmin_flow}, and \eqref{eq:maxmin_const_c} with $\beta=0$ via linear programming (LP), and performance was evaluated as the error relative to these LP-optimal values. The goal of this experiment is to verify that our algorithm, which simultaneously accounts for both max-min fairness and constraint satisfaction, achieves near-optimal performance compared to methods that lack either max-min fairness or constraint enforcement. We compare our constrained max-min algorithm against unconstrained max-min, constrained max-average which does not enforce max-min fairness, and unconstrained max-average baselines.

\begin{table}[ht!]
    \centering
    \caption{Optimal value error in tabular constrained MOMDPs}
    \label{tab:ablation_tabular_weight}
    \setlength{\tabcolsep}{5pt}
    \begin{tabular}{l c}
        \toprule
        Algorithm & Optimal value error $(\downarrow)$ \\
        \midrule
        \textbf{Constrained max-min} & $0.004$  \\
        Unconstrained max-min  &  $0.325$   \\ 
        Constrained max-average   &  $0.657$   \\ 
        Unconstrained max-average &  $1.008$  \\
        \bottomrule
    \end{tabular}
\end{table}

\begin{table}[ht!] 
    \centering
    \caption{ Optimal value errors of the constrained max-min MORL algorithm for different values of $\beta$ }
    \label{tab:tabular_two_mdps_ablation}
    \begin{tabular}{|c|c|c|c|c|c|}
    \hline
     $\beta$  & 0.1 & 0.03 & 0.01 & 0.003 & 0.001   \\
    \hline
    Error $(\downarrow)$  & $0.061$ & $0.004$ & $0.009$ & $0.020$ & $0.021$ \\
    \hline
    \end{tabular}
\end{table}

Table~\ref{tab:ablation_tabular_weight} shows that the constrained max-min approach achieves near-optimal values compared to the other methods. Baselines that incorporate neither max-min fairness nor constraint satisfaction exhibit the largest optimal value errors. We further analyze the effect of $\beta$ on convergence. As shown in Table~\ref{tab:tabular_two_mdps_ablation}, $\beta=0.03$ and $0.01$ yield the smallest errors. The slight increase in optimal value error when $\beta < 0.01$ may be attributed to the first term in \eqref{eq:final_convergence_result} to converge more slowly as the iteration count $m \to \text{ITER}$, the maximum number of iterations used in Algorithm~\ref{alg:constrained_morl}.

\subsection{Extension to Applications} \label{subsec:constrained_deep}

In this section, we extend our algorithm to practical applications, demonstrating that our method is not limited to tabular settings. The Building (Section \ref{subsubsec:building}) and MoAnt (Section \ref{subsubsec:moant}) environments have continuous state and action spaces, while the Traffic environment (Section \ref{subsubsec:traffic}) has a continuous state space and a discrete action space. To ensure stable gradient estimation of our algorithm in continuous state spaces, we parameterize a gradient network $g_\theta(s) \in \mathbb{R}^{L+K}$ to estimate $\nabla_u v^{*}_{u, w}(s)$ and $\nabla_w v^{*}_{u, w}(s)$, following Theorem \ref{thm1_differentiable}. Implementation details for deep neural network policies, including gradient estimation and our constrained max-min algorithm for practical applications, are provided in Appendix \ref{append:realloc_implement}. Our sourcecode can be found in \url{https://github.com/Giseung-Park/Constrained-Maxmin-MORL}.

\subsubsection{Building Thermal Control}\label{subsubsec:building}

We consider a thermal control environment for a building from SustainGym \citep{yeh2023sustaingym}. The environment models a large commercial building with $N_{\text{zone}}=23$ thermal zones, where the system controls the zone temperatures under changing outside conditions and occupancy. At each timestep, the system observes the current temperatures of the $N_{\text{zone}}$ zones together with additional indicators (e.g., outdoor temperature, global horizontal irradiance, ground temperature, occupancy power, carbon intensity, and electricity price). Then the system outputs a continuous action that controls the heating/cooling power of each zone.

The zones are grouped into $K=3$ categories corresponding to the bottom, middle, and top floors of the building. The unconstrained reward is a $K$-dimensional vector, where each component quantifies the comfort of the corresponding group based on the proximity of the zone temperatures to their target values. The cost is defined as the total power consumption of the system, normalized by the environment. The system’s objective is to control the building so as to maximize the minimum cumulative comfort across all groups, while satisfying the energy constraint with its designed threshold value $C_{th}=180$.

We consider five baselines:
(i) a random policy that samples a continuous action at each timestep (Random),
(ii) unconstrained max-average SAC (MA-SAC) \citep{haarnoja18sac},
(iii) max-average SAC with a Lagrangian relaxation (MA-SAC-L) \citep{ha2020learning,yang21wcsac},
(iv) unconstrained max-min MORL algorithm with Gaussian smoothing (Max-min GS) \citep{park24maxmin}, and (v) unconstrained max-min MORL algorithm based on a two-player zero-sum game (ARAM) \citep{byeon2025multiobjectivereinforcementlearningmaxmin}. Each of the baselines lacks either max-min fairness ((iii)), constraint handling ((iv),(v)), or both ((i), (ii)). We report the mean performance computed across five random seeds. (See Appendices \ref{append:modify_maxmin} and \ref{append:building_overview} for the implementation of the max-min GS baseline and for details on the environment and the hyperparameters, respectively.)

\begin{table}[ht!]
    \centering
    \caption{Cumulative cost sum and worst-group comfort return in the building thermal control environment, with the two constraint-satisfying algorithms highlighted in \textbf{bold}}
    \label{tab:results_cost_max_queue_length_sum}
    \begin{tabular}{l c c}
        \toprule
        Algorithm &
        \makecell{Cost sum \\ $(C_{th}=180)$} &
        \makecell{Minimum \\ return $(\uparrow)$} \\
        \midrule
        Random     & $227.8$        & $433.9$  \\
        MA-SAC     & $203.7$        & $641.5$  \\
        MA-SAC-L   & $\textbf{171.4}$ & $620.9$  \\
        Ours       & $\textbf{178.7}$ & $639.8$  \\
        Max-min GS & $202.1$        & $653.6$  \\
        ARAM       & $276.9$        & $664.3$  \\
        \bottomrule
    \end{tabular}
\end{table}

\begin{table}[ht!]
    \centering
    \caption{Ablation study on the impact of weight learning}
    \label{tab:results_edge_ablation}
    \begin{tabular}{l c c}
        \toprule
        Algorithm &
        \makecell{Cost sum \\ $(C_{th}=180)$} &
        \makecell{Minimum \\ return $(\uparrow)$} \\
        \midrule
        Ours             & $178.7$ & $639.8$ \\
        w/o $w$ update   & $178.1$ & $626.7$ \\
        w/o $u$ update   & $200.7$ & $653.5$ \\
        w/o $(u,w)$ upd. & $222.0$ & $646.9$ \\
        \bottomrule
    \end{tabular}
\end{table}

Table \ref{tab:results_cost_max_queue_length_sum} presents the cumulative cost sum and the worst-group comfort return. Compared to the Random baseline, MA-SAC increases the worst-group comfort return but still fails to satisfy the power consumption constraint, with its cost sum exceeding the threshold $C_{th}=180$. While MA-SAC-L satisfies the power constraint, it does so at the cost of a lower worst-group comfort return compared to MA-SAC. However, our method achieves a higher worst-group comfort return relative to MA-SAC-L, while still adhering to the power constraint. We note that both Max-min GS and ARAM violate the power constraint.

Table \ref{tab:results_edge_ablation} shows that removing the max-min-related $w$ update reduces the worst-group comfort return, while ablating the constraint-related $u$ update causes constraint violations. These results confirm that our method effectively balances max-min performance with constraint satisfaction. (For further analysis of the effect of $\beta$ and the computational complexity, see Appendix \ref{append:building_ablation}.)

\subsubsection{Multi-objective Locomotion Control}\label{subsubsec:moant}

\begin{table}[ht!]
  \centering
  \caption{MoAnt-v5 results over five seeds, with the two constraint-satisfying algorithms highlighted in \textbf{bold}}
  \label{tab:moantv5_results}
  \begin{tabular}{l c c}
    \toprule 
    Algorithm
    & \makecell{Cost sum \\ $(C_{th}=50)$}
    & \makecell{Minimum \\ return $(\uparrow)$} \\
    \midrule
    Random      & 146.5           & 48.2 \\
    MA-SAC      & 275.3           & 98.8 \\  
    MA-SAC-L    & \textbf{47.8}   & 83.0 \\  
    Ours        & \textbf{28.3}   & 92.2 \\  
    Max-min GS  & 111.7           & 92.7 \\  
    ARAM        & 620.7           & 101.3 \\
    \bottomrule    
  \end{tabular}
\end{table}

We include MoAnt-v5 environment \citep{felten_toolkit_2023}, where the agent learns locomotion to maximize $x$ and $y$ velocities while keeping energy consumption under a threshold. We consider an asymmetric case where movement in the $x$ direction is attenuated by friction at rate 0.3. The velocities $(v_x, v_y)$, combined with bonus terms, constitute a two-dimensional reward, while the control cost is treated as a constraint. (See Appendix \ref{append:hyperparam_ant} for details on hyperparameters.)

Table \ref{tab:moantv5_results} shows that both our method and MA-SAC-L satisfy the constraints, but our method achieves superior max-min performance. In contrast, the other four algorithms severely violate the constraints, as they do not explicitly account for constraint satisfaction. Overall, our algorithm balances constraint satisfaction and max-min fairness. (See Appendix \ref{append:ant_different_c} for the result with $C_{th}=40$.)

\subsubsection{Greenhouse-gas-emission-aware Traffic
Management}\label{subsubsec:traffic}

We extend our method to an environment with a larger objective space. We note that MORL benchmark environments, particularly those with more than four objectives, remain limited \citep{hayes22survey}. Recently, \citet{park2025reward} proposed a traffic signal control environment with a 16-dimensional objective vector to evaluate scalability in the MORL setting. We modify this environment to incorporate a setting aimed at reducing greenhouse-gas emissions while maintaining max-min fairness.

In a simulated 16-lane four-way intersection, the agent receives a 37-dimensional continuous state encoding traffic information at each timestep. The agent manages the traffic lights by choosing the traffic signal phase as its action. The feedback signal consists of a 16-dimensional reward vector, where each component represents the negative total waiting time for a corresponding lane. Following prior work \citep{park2025reward}, we evaluate performance in an asymmetric traffic flow scenario. Vehicles enter from each direction according to predefined inflow proportions and may proceed straight or make left or right turns.

\begin{table}[ht!]
  \centering
  \caption{Traffic signal control results over five seeds, with the two constraint-satisfying algorithms highlighted in \textbf{bold}}
  \label{tab:tlc16d_results}
  \begin{tabular}{l c c}
    \toprule
    Algorithm
    & \makecell{Cost sum \\ $(C_{th}=70{,}000)$}
    & \makecell{Minimum \\ return $(\uparrow)$} \\
    \midrule
    Random      & 82{,}698  & $-27{,}523$ \\
    MA-PGO      & 77{,}757  & $-20{,}434$ \\
    MA-CPGO     & \textbf{67{,}887} & $-27{,}830$ \\
    Ours        & \textbf{69{,}147} & \textbf{$-25{,}229$} \\
    Max-min GS  & 73{,}162  & $-21{,}527$ \\
    ARAM        & 88{,}748  & $-19{,}700$ \\
    \bottomrule
  \end{tabular}
\end{table}

Our goal is to achieve fair traffic flow across all lanes while enforcing a constraint on total $\isotope{CO}_2$ emissions, contributing to a more sustainable traffic control system. Since this environment operates with a discrete action space, we replaced the MA-SAC and MA-SAC-L baselines with policy-gradient-optimization (PGO) variants, namely MA-PGO and MA-CPGO, respectively. (See Appendix \ref{append:traffic} for implementation details.)

As shown in Table~\ref{tab:tlc16d_results}, both our method and MA-CPGO satisfy the constraints; however, our approach achieves better max-min performance. In contrast, the remaining four algorithms exhibit constraint violations. Overall, our method consistently balances max-min fairness with constraint satisfaction.

\section{Conclusion}

We have proposed a MORL framework that integrates max-min fairness with constraint satisfaction. Our approach offers flexibility in modeling problems that satisfy fairness and operational constraints. We establish a theoretical foundation and develop an algorithm that achieves a balance between max-min fairness and constraint satisfaction, as demonstrated in our experiments. A discussion of limitations and future work is provided in Appendix \ref{append:limitations}.

\section*{Acknowledgment}

We sincerely thank anonymous reviewers for their valuable feedback during the review process. This work was supported by  the Institute of Information \& Communications Technology Planning \& Evaluation (IITP) under Grant No. RS-2024-00457882 (AI Research Hub Project) and the National Research Foundation of Korea (NRF) grant (2022K1A3A1A31093462) funded by the Korea government (MSIT), and by 
the Israel Ministry of Innovation, Science \& Technology, under Grant No. 1001556423 and the ISF grant 2197/22.

\section*{Impact Statement}

In this work, we propose a MORL algorithm that integrates max-min fairness with constraint satisfaction. First, max-min MORL plays a critical role in promoting fairness across objectives in domains such as traffic management and resource allocation. Unfair results can lead to user dissatisfaction and, in turn, degrade overall system performance, for example, by contributing to traffic congestion. Second, incorporating constraints into RL systems is essential for the responsible development of AI systems, especially given real-world limitations on resources such as electricity, power consumption, and fossil fuels.

Our work advances the goal of sustainable AI by simultaneously incorporating fairness and resource constraints into decision-making. This contrasts to traditional methods that prioritize performance alone, often overlooking concerns of equity and efficient resource use. We believe our framework has the potential to make a meaningful and positive impact on the broader AI community, not only in resource allocation but also in emerging areas such as fair and safe alignment of large language models.

\bibliography{references}
\bibliographystyle{icml2026}

\newpage
\appendix
\onecolumn

\section{ Proof on Optimality Gap } \label{append:suboptimality_gap}

\begin{proof}
With a slight abuse of notation, let $J(\pi):= [J_1(\pi),\cdots,J_{K}(\pi)]^\top \in \mathbb{R}^{K}$ and let $\mathcal{H}(\pi)$ denote the expected cumulative entropy of $\pi$. We express the optimization of \eqref{eq:firstForm} and \eqref{eq:original_maxmin_f_const} with $f=\min$ as follows:
\begin{equation} \label{eq:append_reg_prob_pi}
    \max_{ \pi \in \Pi_{\text{feas}} } \min_{1 \leq k \leq K} J_k(\pi) + \beta \mathcal{H}(\pi)
\end{equation}
where $\Pi_{\text{feas}} := \left\{ \pi \,\middle|\, \mathbb{E}_{\mu_0, \pi} \left[ \sum_{t=0}^\infty \gamma^t c_t^{(l)} \right] \ge C^{(l)}, \quad \forall l = 1, \dots, L \right\}$ and it is assumed to be non-empty under the typical assumption in constrained RL \citep{tessler2018reward,ha2020learning}.

Let the optimal solution to the regularized problem in \eqref{eq:append_reg_prob_pi} be $\pi_r^*:= \arg \max_{ \pi \in \Pi_{\text{feas}} } \min_{1 \leq k \leq K} J_k(\pi) + \beta \mathcal{H}(\pi) = \arg \max_{ \pi \in \Pi_{\text{feas}} } \min_{w} \langle w, J(\pi) \rangle + \beta \mathcal{H}(\pi)$ where $\min_{w \in \Delta^K}$ is abbreviated as $\min_w$ for brevity. Let $w^*(\pi) := \arg \min_w \langle w,J(\pi)\rangle$ and $w_r^* := w^*(\pi_r^*)$.

Let the optimal solution to the unregularized problem be $\pi^*:=\arg \max_{ \pi \in \Pi_{\text{feas}} } \min_{w} \langle w, J(\pi) \rangle$ and $w^* = w^*(\pi^*)$. Let the optimal max-min value of the unregularized problem be $V_{w^*}^{\pi^*} := \langle w^*,J(\pi^*)\rangle$. Similarly, let the optimal value of the regularized problem be $V_{w_r^*}^{\pi_r^*} := \langle w_r^*,J(\pi_r^*)\rangle$. For simplicity, we abbreviate $\max_{ \pi \in \Pi_{\text{feas}} }$ as $\max_\pi$ below.

First, we have the following:
\begin{equation}
    V_{w_r^*}^{\pi_r^*} + \beta {\mathcal{H}}(\pi_r^*) = \max_\pi\min_w \langle w,J(\pi)\rangle+\beta \mathcal{H}(\pi) \ge \min_w \langle w,J(\pi^*)\rangle+\beta  \mathcal{H}(\pi^*) =\langle w^*,J(\pi^*)\rangle+\beta  \mathcal{H}(\pi^*) = V_{w^*}^{\pi^*} +\beta  \mathcal{H}(\pi^*).
\end{equation}

Since the stepwise entropy of $\pi$ is upper-bounded by $\log |\mathcal{A}|$, we have $0\le \mathcal{H}(\pi)\le \frac{\log|\mathcal{A}|}{1-\gamma}$ and obtain 

\begin{equation} \label{eq:append1_upper_bound}
    V_{w^*}^{\pi^*} -V_{w_r^*}^{\pi_r^*}  \le  \frac{\beta\log|\mathcal{A}|}{1-\gamma}.
\end{equation}

Similarly, we have the following:
\begin{equation} \label{eq:append1_lower_bound}
    V_{w^*}^{\pi^*} = \max_\pi\min_w \langle w,J(\pi)\rangle \ge \min_w \langle w,J(\pi_r^*)\rangle =\langle w^*_r,J(\pi_r^*)\rangle = V_{w_r^*}^{\pi_r^*}.
\end{equation}

Combining \eqref{eq:append1_upper_bound} and \eqref{eq:append1_lower_bound}, we obtain that the optimality value gap ranges as $0 \le V_{w^*}^{\pi^*} - V_{w_r^*}^{\pi_r^*} \le \frac{\beta\log|\mathcal{A}|}{1-\gamma}$.
\end{proof}

\newpage
\section{Proof on Equivalent Optimization} \label{append:eqiv_conv_opt}

\begin{proof}
The dual problem of \eqref{eq:maxmin_obj}, \eqref{eq:maxmin_flow}, and \eqref{eq:maxmin_const_c}  is rewritten as follows:
    \begin{align} \label{eq:dual_of_maxent_maxmin_const}
         & {\min_{u \geq 0}} \min_{w\geq 0, v} \hspace{-0.1em}\min_{\xi \geq 0} \hspace{-0.1em}\max_{\rho,b} \hspace{-0.2em}\Bigg[ b(1 - \hspace{-0.1em}\sum_{k=1}^K w_k) \hspace{-0.1em}-\hspace{-0.1em} \beta \hspace{-0.1em}\sum_{s,a}\hspace{-0.1em} \rho(s,a) \hspace{-0.1em}\log \frac{\rho(s,a)}{\sum_{a'} \rho(s,a')}    \nonumber \\
          & + \hspace{-0.1em}\sum_{s}\hspace{-0.1em}\mu_0(s) v(s) + \sum_{s,a} \xi(s,a)\rho(s,a) {- \sum_{l=1}^L u_l C^{(l)}} \nonumber \\
          & + \hspace{-0.1em}\sum_{s,a} \rho(s,a)\hspace{-0.1em} [\sum_{k=1}^K w_k r^{(k)}\hspace{-0.1em}(s,a) \hspace{-0.1em} {+\sum_{l=1}^L u_l c^{(l)}(s,a)} +\hspace{-0.1em} \gamma \hspace{-0.1em}\sum_{s'} T(s' | s,a) v(s') \hspace{-0.1em}-\hspace{-0.1em} v(s)  ] \Bigg].
    \end{align}

Here $b$ is an auxiliary variable satisfying $\sum_{s,a} r^{(k)}(s,a) \rho(s,a) \geq b, ~ 1 \leq k \leq K$. Let $\eta_{u, v,w}(s,a) := \sum_{k=1}^K w_k r^{(k)}(s,a) +\sum_{l=1}^L u_l c^{(l)}(s,a) + \gamma \sum_{s'} T(s' | s,a) v(s') - v(s)$. We apply KKT conditions.

1. Stationarity condition gives
\begin{equation} \label{eq:stationarity_rho}
    \forall (s,a), ~~ - \beta \log \frac{\rho(s,a)}{\sum_{a'} \rho(s,a')}  + \xi(s,a) + \eta_{u,v,w}(s,a) = 0
\end{equation}
and
\begin{equation} \label{eq:w_sum_to_1}
    1 - \sum_{k=1}^K w_k = 0.
\end{equation}

2. Complementary slackness condition gives 
\begin{equation} \label{eq:slackness}
    \forall (s,a), ~~ \xi(s,a) \rho(s,a) = 0.
\end{equation}

From \eqref{eq:stationarity_rho}, we derive 
\begin{equation} \label{eq:kkt_intermediate_const}
    \forall (s,a), ~~ \frac{\rho(s,a)}{\sum_{a'} \rho(s,a')} = \exp \left( \frac{\xi(s,a) + \eta_{u,v,w}(s,a)}{\beta} \right) 
\end{equation}
so $\rho(s,a) > 0$ and $\xi(s,a) = 0$ from \eqref{eq:slackness}.  Therefore,
\begin{equation} \label{eq:optimal_policy_append}
    \forall (s,a), ~~ \frac{\rho(s,a)}{\sum_{a'} \rho(s,a')} = \exp \left( \frac{\eta_{u,v,w}(s,a)}{\beta} \right). 
\end{equation}

Inserting \eqref{eq:w_sum_to_1} and \eqref{eq:optimal_policy_append}, we obtain:
\begin{equation} \label{eq:joint_dual_maxent_maxmin_obj_const}
    \min_{u \in \mathbb{R}_{+}^L} \min_{v,w} \sum_s \mu_0(s) v(s) - \sum_{l=1}^L u_l C^{(l)}
\end{equation}
\begin{equation} \label{eq:joint_dual_maxent_maxmin}
     \forall s, ~~ v(s) \hspace{-0.2em} = \hspace{-0.2em}\beta \log \hspace{-0.2em}\sum_a \hspace{-0.2em}\exp [ \frac{1}{\beta}  \{ \sum_{k=1}^K \hspace{-0.2em}w_k r^{(k)}\hspace{-0.2em}(s,a) +\sum_{l=1}^L u_l c^{(l)}(s,a)  + \gamma \hspace{-0.3em}\sum_{s'}\hspace{-0.2em}T(s'|s,a)v(s') \} ] := [\mathcal{T}_{u,w} v](s)
\end{equation}
\begin{equation} 
    \sum_{k=1}^K w_k = 1; ~ w_k \geq 0 ~~ \forall 1 \leq k \leq K.
\end{equation}
where \eqref{eq:joint_dual_maxent_maxmin} is derived from $\sum_a \exp \left( \frac{\eta_{u,v,w}(s,a)}{\beta} \right) = 1, ~~ \forall s$, and strong duality holds under Slater condition \citep{boyd2004convex}. Since $\mathcal{T}_{u,w}$ is a contraction mapping \citep{fox2016taming,haarnoja2017sql}, it has the unique fixed point $v^*_{u,w}$. Therefore, $v = v^*_{u,w}$ is the only feasible solution  that satisfies \eqref{eq:joint_dual_maxent_maxmin} and we have the following:
\begin{equation} 
    \min_{u \in \mathbb{R}_{+}^L, w \in \Delta^K} \mathcal{L}(u,w) = \sum_s \mu_0(s) v^{*}_{u,w}(s) - \sum_{l=1}^L u_l C^{(l)}.
\end{equation}
Under Slater condition, this optimization attains the same optimal value as in the original convex optimization. The convexity of $v^{*}_{u,w}(s)$ follows from the proof of Theorem 4.1 in \citet{park24maxmin}.

\end{proof}

\section{ Proof of Differentiability } \label{append:differentiable}

\begin{proof}
    For notational simplicity, we first present the theorem for the case $L=0$ (i.e., without constraints) and then extend the result to the case  $L \geq 1$. This holds because, given $(u,w) \in \mathbb{R}^{L+K}$, the mapping $\mathcal{T}_{u,w}$ is defined by $[\mathcal{T}_{u,w} v](s) = \beta \log \sum_a \exp [ \frac{1}{\beta} \{ \sum_{l=1}^L u_l c^{(l)}(s,a)  + \sum_{k=1}^K w_k r^{(k)}(s,a) + \gamma \hspace{-0.3em}\sum_{s'}\hspace{-0.2em}T(s'|s,a) v(s') \} ], \forall s$, and we can regard the concatenation of $c(s,a)$ and $r(s,a)$ as a new vector reward of size $L+K$ with its weight $(u,w)$. Therefore, we use the notation of the following mapping $[\mathcal{T}_{w} v](s) = \beta \log \sum_a \exp [ \frac{1}{\beta} \{ \sum_{k=1}^K w_k r^{(k)}(s,a) + \gamma \hspace{-0.3em}\sum_{s'}\hspace{-0.2em}T(s'|s,a) v(s') \} ], \forall s$. 

    Let $|\mathcal{S}|=p$. We define $F(w,v):=v-\mathcal{T}_w v$, $F: \mathbb{R}^{K} \times \mathbb{R}^{p} \to \mathbb{R}^{p}$. Let $v^*_w$ be the unique fixed point of $\mathcal{T}_w$. Then $F(w,v^*_w)=v^*_w-\mathcal{T}_w v^*_w = 0$. Here $v^*_w$ is implicitly expressed with respect to (w.r.t.) $w$, and we aim to analyze $v^*_w$ using \textit{implicit function theorem}.

    First of all, $F: \mathbb{R}^{K} \times \mathbb{R}^{p} \to \mathbb{R}^{p}$ is a continuously differentiable function. For each $s$, $F(w,v)(s)=v(s)- [\mathcal{T}_w v](s) = v(s) - \beta \log \sum_a \exp [ \frac{1}{\beta} \{ \sum_{k=1}^K w_k r^{(k)}(s,a) + \gamma \hspace{-0.3em}\sum_{s'}\hspace{-0.2em}T(s'|s,a) v(s') \} ]$ which is a composition of linear, logarithm, summation, exponential, and linear functions.

    Now we fix $w$ and check whether the Jacobian matrix $\partial_v F(w,v)|_{v = v^{*}_w} \in \mathbb{R}^{p \times p}$ is invertible where $[\partial_v F(w,v)|_{v = v^{*}_w}]_{ij} = \frac{\partial F(w,v)(s_i)}{\partial v(s_j)}|_{v = v^{*}_w}$. We have $\partial_v F(w,v) = I_p -\partial_v [\mathcal{T}_w v]$ where $I_p$ is the $p \times p$ identity matrix. Then
    \begin{equation}
         \frac{\partial [\mathcal{T}_w v](s_i)}{\partial v(s_j)}|_{v = v^{*}_w} = \gamma \mathbb{E}_{a \sim \pi^*_w(\cdot | s_i)}[T(s_j | s_i, a)]
    \end{equation}
    where
    \begin{equation} \label{eq:pi_star_w_only}
        \pi^*_w(a | s) = \frac{ \exp [ \frac{1}{\beta} \{ \sum_{k=1}^K w_k r^{(k)}(s,a) + \gamma \hspace{-0.3em}\sum_{s'}\hspace{-0.2em}T(s'|s,a) v^{*}_w(s') \} ]}{\sum_{a'} \exp [ \frac{1}{\beta} \{ \sum_{k=1}^K w_k r^{(k)}(s,a') + \gamma \hspace{-0.3em}\sum_{s'}\hspace{-0.2em}T(s'|s,a') v^{*}_w(s') \} ]}.
    \end{equation} 

    If we denote $T(\cdot |s,a) := [T(s_1|s,a)\cdots T(s_p|s,a)]$, we have
    \begin{equation}
        \partial_v F(w,v)|_{v = v^{*}_w} =I_p -\gamma\begin{bmatrix}
        \mathbb{E}_{a \sim \pi^*_w(\cdot | s_1)}[T(\cdot | s_1, a)] \\
        \vdots\\
        \mathbb{E}_{a \sim \pi^*_w(\cdot | s_p)}[T(\cdot | s_p, a)]
        \end{bmatrix} =: I_p -\gamma\begin{bmatrix}
        T^{\pi^*_w}(\cdot | s_1) \\
        \vdots\\
        T^{\pi^*_w}(\cdot | s_p)
        \end{bmatrix}
    \end{equation}
    where $T^{\pi^*_w}(s_j | s_i) = \mathbb{E}_{a \sim \pi^*_w(\cdot | s_i)}[T(s_j | s_i, a)] =: [T^{\pi^*_{w}}]_{ij}$. Then $I_p - \gamma T^{\pi^*_w}$ is invertible since $T^{\pi^*_w}$ is a row-stochastic square matrix \citep{horn2012matrix}. 

    Therefore, $\partial_v F(w,v)|_{v = v^{*}_w}$ is invertible. By implicit function theorem, there exists an open set $U\subset\mathbb{R}^K$ containing $w$ such that there exists a unique continuously differentiable function $h:U\to\mathbb{R}^{p}$ such that $h(w)=v^{*}_w$ and $F(w',h(w'))=0$, i.e., $h(w')=\mathcal{T}_{w'} h(w')$ for all $w'\in U$. Since  $h(w')$ is the unique fixed point of $\mathcal{T}_{w'}$, $h(w') = v^{*}_{w'}, \forall w'\in U$. If we use the implicit function theorem for all $w \in \mathbb{R}^K$, we can conclude that $v = v^*_w$ is a unique continuously differentiable function in $w \in \mathbb{R}^K$ that satisfies $v = \mathcal{T}_w v$.
    
    Moreover, for $1 \leq k \leq K$,
    \begin{equation}
         \frac{\partial [\mathcal{T}_w v](s_i)}{\partial w_k}|_{v = v^{*}_w} = \mathbb{E}_{a \sim \pi^*_w(\cdot | s_i)}[r^{(k)}(s_i, a)].
    \end{equation}
    With a slight abuse of notation, if we denote $r(s,a) := [r^{(1)}(s,a)\cdots r^{(K)}(s,a)]$, we have
    \begin{equation}
        \partial_w F(w,v)|_{v = v^{*}_w} = -\begin{bmatrix}
        \mathbb{E}_{a \sim \pi^*_w(\cdot | s_1)}[ r(s_1, a)] \\
        \vdots\\
        \mathbb{E}_{a \sim \pi^*_w(\cdot | s_p)}[r(s_p, a)]
        \end{bmatrix} =: -\begin{bmatrix}
        r^{\pi^*_w}(s_1) \\
        \vdots\\
        r^{\pi^*_w}(s_p)
        \end{bmatrix}
    \end{equation}
    where $r^{\pi^*_w}(s) = \mathbb{E}_{a \sim \pi^*_w(\cdot | s)}[ r(s, a)] \in \mathbb{R}^{1 \times K}$. By implicit function theorem, we have    
    \begin{equation} \label{eq:grad_all}
        \begin{bmatrix}
    \nabla_{w} v^*_{w}(s_1)^\top \\
    \vdots\\
    \nabla_{w} v^*_{w}(s_p)^\top
\end{bmatrix} = - [\partial_v F(w,v)|_{v = v^{*}_w}]^{-1} \partial_w F(w,v)|_{v = v^{*}_w} = (I_p - \gamma T^{\pi^*_w})^{-1} r^{\pi^*_w}.
    \end{equation}
    Note that the $k$-th ($1 \leq k \leq K$) column of \eqref{eq:grad_all} is equivalent to the policy evaluation of $\pi^*_w$ considering a scalar reward function $r^{(k)}$ \citep{silver2015,sutton2018reinforcement}. We denote the value function as $v^{\pi^*_w}_k \in \mathbb{R}^p$. Then
    \begin{equation}
        \frac{\partial v^*_w (s)}{\partial w_k} = v^{\pi^*_w}_k(s), ~ \forall s.
    \end{equation}
    If we denote $v^{\pi^*_w}(s) = [v^{\pi^*_w}_1(s), \cdots, v^{\pi^*_w}_K(s)]^\top \in \mathbb{R}^K$ for all $s$, then $v^{\pi^*_w}(s)$ is the value function evaluated with the policy $\pi^*_w$ in a given MOMDP. We have
    \begin{equation}
        \nabla_{w} v^*_{w}(s) = v^{\pi^*_w}(s), ~ \forall s.
    \end{equation}
    For the case of $L \geq 1$, the only difference is that $\pi^*_{w}$ is changed to 
    \begin{equation}
        \pi^*_{u,w}(a | s) = \frac{ \exp [ \frac{1}{\beta} \{ \sum_{l=1}^L u_l c^{(l)}(s,a) + \sum_{k=1}^K w_k r^{(k)}(s,a) + \gamma \hspace{-0.3em}\sum_{s'}\hspace{-0.2em}T(s'|s,a) v^{*}_{u,w}(s') \} ]}{\sum_{a'} \exp [ \frac{1}{\beta} \{ \sum_{l=1}^L u_l c^{(l)}(s,a') + \sum_{k=1}^K w_k r^{(k)}(s,a') + \gamma \hspace{-0.3em}\sum_{s'}\hspace{-0.2em}T(s'|s,a') v^{*}_{u,w}(s') \} ]}
    \end{equation} 
    where $v^{*}_{u, w}$ is the fixed point of the operator $\mathcal{T}_{u,w}$: 
    \begin{equation} 
        \forall s, ~ [\mathcal{T}_{u,w} v](s) = \beta \log \sum_a \exp [ \frac{1}{\beta} \{ \sum_{l=1}^L u_l c^{(l)}(s,a) + \sum_{k=1}^K w_k r^{(k)}(s,a) + \gamma \hspace{-0.3em}\sum_{s'}\hspace{-0.2em}T(s'|s,a) v(s') \} ]
    \end{equation}
    and the column size of $r^{\pi^*_{u,w}}$ is $L+K$, not $K$. We denote $v_c^{\pi^*_{u,w}}(s) \in \mathbb{R}^L$, $v_r^{\pi^*_{u,w}}(s) \in \mathbb{R}^K$ as the value functions evaluated with the policy $\pi^*_{u,w}$ for constrained reward $c$ and unconstrained reward $r$, respectively. Finally, we have
    \begin{equation}
        \nabla_{u} v^*_{u, w}(s) = v_c^{\pi^*_{u,w}}(s), ~ \nabla_{w} v^*_{u, w}(s) = v_r^{\pi^*_{u,w}}(s), ~ \forall s.
    \end{equation}
\end{proof}

\section{ Proof of Smoothness } \label{append:thm_smooth}

\begin{proof}
    Let $a = (u',w')$ and $b = (u'', w'')$ in $\mathbb{R}^{L+K}$ and $a \neq b$. By the differentiability of $\nabla v^*_{u,w}(s)$ for each $s$ in Proposition \ref{thm2_twice_differentiable}, $\nabla v^*_{u,w}(s)$ is continuous on $\mathbb{R}^{L+K}$. Specifically, $\nabla v^*_{u,w}(s)$ is continuous on $\{a + t(b-a) | 0 \leq t \leq 1  \}$ and differentiable on $\{a + t(b-a) | 0 < t < 1  \}$. Therefore, by the generalized mean value inequality, we have
    \begin{equation} \label{eq:general_mean_ineq}
        \| \nabla v^*_{u,w}(s)|_{(u,w) = b} - \nabla v^*_{u,w}(s)|_{(u,w) = a}  \|_2 \leq \sup_{t \in [0,1] }\| H[v^*_{u,w}(s)]|_{(u,w) = a + t(b-a)} \|_2  \| b - a \|_2.
    \end{equation}
    where $H[v^*_{u,w}(s)]$ is the Hessian of $v^*_{u,w}(s)$. Let $\lambda_{\text{max}}( A )$ be the maximum eigenvalue of a real symmetric matrix $A$. For each $s_k ~ (1 \leq k \leq p)$, the eigenvalues of $H[v^*_{u,w}(s_k)]$ are nonnegative because of its positive semidefiniteness in \eqref{eq:hessian} in Proposition \ref{thm2_twice_differentiable}. Since trace operator is additive, we have
    \begin{equation}
        \| H[v^*_{u,w}(s_k)] \|_2 = \lambda_{\text{max}}( H[v^*_{u,w}(s_k)] ) \leq \text{Tr}( H[v^*_{u,w}(s_k)] ) = \frac{1}{\beta} \sum_{l=1}^p [(I_p - \gamma T^{\pi^*_{u,w}})^{-1}]_{kl} \text{Tr}(B^{\pi^*_{u,w}}(s_l)).
    \end{equation}
     For each $s$, we also have
    \begin{equation}
        \text{Tr}( B^{\pi^*_{u,w}}(s))   = \sum_{k=1}^{L+K} \text{Var}(Q_k^{\pi^*_{u,w}}(s,a)) \leq \sum_{k=1}^{L+K} \mathbb{E}[ | Q_k^{\pi^*_{u,w}}(s,\cdot) |^2 ] \leq \sum_{k=1}^{L+K} \left( \frac{r^{(k)}_{\text{max}}}{1-\gamma} \right)^2
    \end{equation}
    where $Q_k^{\pi^*_{u,w}}(s,a)$ is the $k$-th dimension  of $Q^{\pi^*_{u,w}}(s,a) \in \mathbb{R}^{K+L}$.
      
    Since $(I_p - \gamma T^{\pi^*_{u,w}})^{-1} = \sum_{i=0}^\infty (\gamma T^{\pi^*_{u,w}})^i$ and each $(T^{\pi^*_{u,w}})^i$ is a probability transition matrix,
    \begin{equation}
        \| H[v^*_{u,w}(s_k)] \|_2 \leq \frac{1}{\beta} \sum_{m=1}^{L+K} \left( \frac{r^{(m)}_{\text{max}}}{1-\gamma} \right)^2 \left( \sum_{i=0}^\infty \gamma^i \sum_{l=1}^p (T^{\pi^*_{u,w}})^i_{kl} \right) = \frac{1}{\beta (1-\gamma)} \sum_{m=1}^{L+K} \left( \frac{r^{(m)}_{\text{max}}}{1-\gamma} \right)^2.
    \end{equation}
    Here, the summations of $ \sum_{l=1}^p$ and $\sum_{i=0}^\infty $ are interchangable due to the absolute summability. Finally, we have $\| H[\mathcal{L}(u,w)]\|_2 \leq \sum_s \mu_0(s) \| H[v^*_{u,w}(s)] \|_2 \leq \frac{1}{\beta (1-\gamma)} \sum_{m=1}^{L+K} \left( \frac{r^{(m)}_{\text{max}}}{1-\gamma} \right)^2 =: \alpha$.   
    
    It should be noted that $\| H[v^*_{u,w}(s_k)] \|_2$ is uniformly bounded regardless of $s_k$ and $(u,w)$.  Therefore, from \eqref{eq:general_mean_ineq}, $\nabla \mathcal{L}(u,w)$ is $\alpha$-Lipschitz continuous in $\| \cdot \|_2$.
\end{proof}

\section{ Proof of Twice-Differentiability } \label{append:twice_differentiable}

\begin{proof}
    Here we also use the implicit function theorem and follow a similar logic in the proof of differentiability in Appendix \ref{append:differentiable}. Let $|\mathcal{S}| = p$. We show the theorem for the case of $L=0$ to guarantee notational simplicity. For each $1 \leq i \leq K$, we want to show that $\frac{\partial v^*_w}{\partial w_i} := [\frac{\partial v^*_w(s_1)}{\partial w_i}, \cdots, \frac{\partial v^*_w(s_p)}{\partial w_i}]^\top \in \mathbb{R}^{p}$ is differentiable in $w \in \mathbb{R}^K$. From the result in Appendix \ref{append:differentiable}, we have
    \begin{equation}
        \frac{\partial v^*_w}{\partial w_i} = v^{\pi^*_w}_i
    \end{equation}
    where $v^{\pi^*_w}_i \in \mathbb{R}^p$ is the value function evaluated with the policy $\pi^*_w$ in \eqref{eq:pi_star_w_only} with the $i$-th reward $r^{(i)}$. Let $r_i^{\pi^*_w}(s) = \mathbb{E}_{a \sim \pi^*_w(\cdot | s)}[ r^{(i)}(s, a)] \in \mathbb{R}$. From \eqref{eq:grad_all}, we have
    \begin{equation}
        v^{\pi^*_w}_i = (I_p - \gamma T^{\pi^*_w})^{-1} r_i^{\pi^*_w}
    \end{equation}
    or equivalently,
    \begin{equation}
        v^{\pi^*_w}_i = r_i^{\pi^*_w} + \gamma T^{\pi^*_w} v^{\pi^*_w}_i =: \mathcal{T}_w^* v^{\pi^*_w}_i.
    \end{equation}

    We define $F(w,v):=v-\mathcal{T}_w^* v$, $F: \mathbb{R}^{K} \times \mathbb{R}^{p} \to \mathbb{R}^{p}$. Then $F(w,v^{\pi^*_w}_i)=v^{\pi^*_w}_i-\mathcal{T}_w v^{\pi^*_w}_i = 0$. Here $v^{\pi^*_w}_i$ is the unique fixed point of $\mathcal{T}_w^*$ and is implicitly expressed w.r.t. $w$, and we aim to analyze $v^{\pi^*_w}_i$ using \textit{implicit function theorem}.

    First of all, $F: \mathbb{R}^{K} \times \mathbb{R}^{p} \to \mathbb{R}^{p}$ is a continuously differentiable function. For each $s$, $F(w,v)(s)=v(s)- [\mathcal{T}_w^* v](s) =  v(s) - [ r_i^{\pi^*_w}(s) + \gamma \sum_{s'} T^{\pi^*_w}(s'|s) v(s')  ] = v(s) - \sum_a \pi^*_w(a|s)[    r^{(i)}(s,a) + \gamma \sum_{s'} T(s'|s,a) v(s')  ]$. As seen in \eqref{eq:pi_star_w_only}, $\pi^*_w$ contains $v^*_w$ which is continuously differentiable in $w$ (as a result of the proof in Appendix \ref{append:differentiable}), and $\pi^*_w$ is a composition of quotient, exponential, summation and linear functions of $w$ and $v^*_w$. 

    Now we fix $w$ and check whether the Jacobian matrix $\partial_v F(w,v)|_{v = v^{\pi^*_w}_i} \in \mathbb{R}^{p \times p}$ is invertible where $[\partial_v F(w,v)|_{v = v^{\pi^*_w}_i}]_{ij} = \frac{\partial F(w,v)(s_i)}{\partial v(s_j)}|_{v = v^{\pi^*_w}_i}$. We have $\partial_v F(w,v) = I_p -\partial_v [\mathcal{T}_w^* v]$ where $I_p$ is the $p \times p$ identity matrix. Then
    \begin{equation}
         \frac{\partial [\mathcal{T}_w^* v](s_i)}{\partial v(s_j)}|_{v = v^{\pi^*_w}_i} = \gamma \mathbb{E}_{a \sim \pi^*_w(\cdot | s_i)}[T(s_j | s_i, a)].
    \end{equation}

    If we denote $T(\cdot |s,a) := [T(s_1|s,a)\cdots T(s_p|s,a)]$, we have
    \begin{equation}
        \partial_v F(w,v)|_{v = v^{\pi^*_w}_i} =I_p -\gamma\begin{bmatrix}
        \mathbb{E}_{a \sim \pi^*_w(\cdot | s_1)}[T(\cdot | s_1, a)] \\
        \vdots\\
        \mathbb{E}_{a \sim \pi^*_w(\cdot | s_p)}[T(\cdot | s_p, a)]
        \end{bmatrix} =: I_p -\gamma\begin{bmatrix}
        T^{\pi^*_w}(\cdot | s_1) \\
        \vdots\\
        T^{\pi^*_w}(\cdot | s_p)
        \end{bmatrix}
    \end{equation}
    where $T^{\pi^*_w}(s_j | s_i) = \mathbb{E}_{a \sim \pi^*_w(\cdot | s_i)}[T(s_j | s_i, a)] =: [T^{\pi^*_{w}}]_{ij}$. Then $I_p - \gamma T^{\pi^*_w}$ is invertible since $T^{\pi^*_w}$ is a row-stochastic square matrix \citep{horn2012matrix}. 

    Therefore, $\partial_v F(w,v)|_{v = v^{\pi^*_w}_i}$ is invertible. By implicit function theorem, there exists an open set $U\subset\mathbb{R}^K$ containing $w$ such that there exists a unique continuously differentiable function $h:U\to\mathbb{R}^{p}$ such that $h(w)=v^{\pi^*_w}_i$ and $F(w',h(w'))=0$, i.e., $h(w')=\mathcal{T}_{w'}^* h(w')$ for all $w'\in U$. Since  $h(w')$ is the unique fixed point of $\mathcal{T}_{w'}^*$, $h(w') = v^{\pi^*_{w'}}_i, \forall w'\in U$. If we use the implicit function theorem for all $w \in \mathbb{R}^K$, we can conclude that $v = v^{\pi^*_w}_i$ is a unique continuously differentiable function in $w \in \mathbb{R}^K$ that satisfies $v = \mathcal{T}_w^* v$.
    
    Now, for $1 \leq j \leq K$, we aim to calculate $\frac{\partial [\mathcal{T}_w^* v](s)}{\partial w_j}|_{v = v^{\pi^*_w}_i}$. For notational simplicity, let $Q^*_w(s,a) := \sum_{k=1}^K w_k r^{(k)}(s,a) + \gamma \hspace{-0.3em}\sum_{s'}\hspace{-0.2em}T(s'|s,a) v^{*}_w(s')$. Then we express $\pi^*_w$ as follows:
    \begin{equation}
        \pi^*_w(a | s) = \frac{ \exp [ \frac{1}{\beta} \{ Q^*_w(s,a) \} ]}{\sum_{a'} \exp [ \frac{1}{\beta} \{ Q^*_w(s,a') \} ]}.
    \end{equation} 
    We also have
    {\small
    \begin{equation}
        \frac{\partial Q^*_w(s,a)}{\partial w_j} = r^{(j)}(s,a) + \gamma \hspace{-0.3em}\sum_{s'}\hspace{-0.2em}T(s'|s,a) \frac{\partial v^{*}_w(s')}{\partial w_j} = r^{(j)}(s,a) + \gamma \hspace{-0.3em}\sum_{s'}\hspace{-0.2em}T(s'|s,a) v^{\pi^*_w}_j(s') := Q_j^{\pi^*_w}(s,a).
    \end{equation}
    }
    
    In other words, we denote $Q_j^{\pi^*_w}$ as the action-value function evaluated with $\pi^*_w$ for a scalar reward function $r^{(j)}$. Then    
    \begin{equation}
        \frac{\partial [\mathcal{T}_w^* v](s)}{\partial w_j}|_{v = v^{\pi^*_w}_i} = \sum_a Q_i^{\pi^*_w}(s,a) \frac{\partial \pi^*_w(a | s) }{\partial w_j} 
    \end{equation}
    which is equivalent to
    \begin{equation}
        \frac{\partial [\mathcal{T}_w^* v](s)}{\partial w_j}|_{v = v^{\pi^*_w}_i}  = \frac{1}{\beta}  \sum_a Q_i^{\pi^*_w}(s,a) \bigg[ \pi^*_w(a | s) Q_j^{\pi^*_w}(s,a) - \pi^*_w(a | s) \sum_{a'} \{ \pi^*_w(a' | s) Q_j^{\pi^*_w}(s,a') \} \bigg]
    \end{equation}
    and we have
    {\small
    \begin{equation}
        \frac{\partial [\mathcal{T}_w^* v](s)}{\partial w_j}|_{v = v^{\pi^*_w}_i} = \frac{1}{\beta} \bigg[ \mathbb{E}_{a \sim \pi^*_w(\cdot | s)} [ Q_i^{\pi^*_w}(s,a) Q_j^{\pi^*_w}(s,a) ] - \mathbb{E}_{a \sim \pi^*_w(\cdot | s)} [ Q_i^{\pi^*_w}(s,a) ] \mathbb{E}_{a \sim \pi^*_w(\cdot | s)} [ Q_j^{\pi^*_w}(s,a) ] \} \bigg].
    \end{equation}
    }
    
    By implicit function theorem, we have    
    \begin{equation}
        \begin{bmatrix}
    \nabla_{w} \frac{\partial v^*_{w}(s_1)}{\partial w_i}^\top \\
    \vdots\\
    \nabla_{w} \frac{\partial v^*_{w}(s_p)}{\partial w_i}^\top
\end{bmatrix} = - [\partial_v F(w,v)|_{v = v^{\pi^*_w}_i}]^{-1} \partial_w F(w,v)|_{v = v^{\pi^*_w}_i} = \frac{1}{\beta}(I_p - \gamma T^{\pi^*_w})^{-1} E_i^{\pi^*_w}
    \end{equation}
    where $E_i^{\pi^*_w}$ is a $p \times K$ matrix where for each row corresponding to $s$, the $j$-th element is $\mathbb{E}_{a \sim \pi^*_w(\cdot | s)} [ Q_i^{\pi^*_w}(s,a) Q_j^{\pi^*_w}(s,a) ] - \mathbb{E}_{a \sim \pi^*_w(\cdot | s)} [ Q_i^{\pi^*_w}(s,a) ] \mathbb{E}_{a \sim \pi^*_w(\cdot | s)} [ Q_j^{\pi^*_w}(s,a) ] \}$. 
    This formulation holds for each $1 \leq i \leq K$. 
    
    Therefore, we construct a $p \times K \times K$ tensor, say $B^{\pi^*_w}$, by stacking $\{E_i^{\pi^*_w}\}_{i}$ along the new (third) dimension. Then along the first dimension of size $p$, for each $s$, let $B^{\pi^*_w}(s) \in \mathbb{R}^{K \times K}$ be the corresponding slice of $B$. Let $Q^{\pi^*_w}(s,a) = [Q_1^{\pi^*_w}(s,a), \cdots, Q_K^{\pi^*_w}(s,a)]^\top \in \mathbb{R}^K$ be the action-value function evaluated with $\pi^*_w$ for vector reward $r$. Then we have
    
    {\small
    \begin{equation}
        B^{\pi^*_w}(s) = \mathbb{E}_{a \sim \pi^*_w(\cdot|s)} \bigg[ ( Q^{\pi^*_w}(s,a) - \mathbb{E}_{a' \sim \pi^*_w(\cdot|s)} [Q^{\pi^*_w}(s,a')] ) ( Q^{\pi^*_w}(s,a) - \mathbb{E}_{a' \sim \pi^*_w(\cdot|s)} [Q^{\pi^*_w}(s,a')] )^\top    \bigg]
    \end{equation}
    }
    
    which is the covariance matrix of $Q^{\pi^*_w}(s,\cdot)$ over the probability distribution $\pi^*_w(\cdot|s)$. Let $s_k$ correspond to the $k$-th row of $T^{\pi^*_w} ~ (1 \leq k \leq p)$. Then we have the following Hessian formulation for $s_k$:
    \begin{equation}
        H[v^*_{w}(s_k)] = \frac{1}{\beta} \sum_{l=1}^p [(I_p - \gamma T^{\pi^*_w})^{-1}]_{kl} B^{\pi^*_w}(s_l).
    \end{equation}
    
    For the case of $L \geq 1$, the only difference is that $\pi^*_{w}$ is changed to 
    \begin{equation}
        \pi^*_{u,w}(a | s) = \frac{ \exp [ \frac{1}{\beta} \{ \sum_{l=1}^L u_l c^{(l)}(s,a) + \sum_{k=1}^K w_k r^{(k)}(s,a) + \gamma \hspace{-0.3em}\sum_{s'}\hspace{-0.2em}T(s'|s,a) v^{*}_{u,w}(s') \} ]}{\sum_{a'} \exp [ \frac{1}{\beta} \{ \sum_{l=1}^L u_l c^{(l)}(s,a') + \sum_{k=1}^K w_k r^{(k)}(s,a') + \gamma \hspace{-0.3em}\sum_{s'}\hspace{-0.2em}T(s'|s,a') v^{*}_{u,w}(s') \} ]}
    \end{equation} 
    where $v^{*}_{u, w}$ is the fixed point of the operator $\mathcal{T}_{u,w}$: 
    \begin{equation} 
        \forall s, ~ [\mathcal{T}_{u,w} v](s) = \beta \log \sum_a \exp [ \frac{1}{\beta} \{ \sum_{l=1}^L u_l c^{(l)}(s,a) + \sum_{k=1}^K w_k r^{(k)}(s,a) + \gamma \hspace{-0.3em}\sum_{s'}\hspace{-0.2em}T(s'|s,a) v(s') \} ]
    \end{equation}
    and the size of $B^{\pi^*_{u,w}}(s)$ is $(L+K) \times (L+K)$, not $K \times K$, defined by $Q^{\pi^*_{u,w}}(s,a) \in \mathbb{R}^{L+K}$ which is the action-value function evaluated with $\pi^*_{u,w}$ for the concatenated vector function of constrained reward $c$ and unconstrained reward $r$. Finally, we have
    \begin{equation}
        H[v^*_{u,w}(s_k)] = \frac{1}{\beta} \sum_{l=1}^p [(I_p - \gamma T^{\pi^*_{u,w}})^{-1}]_{kl} B^{\pi^*_{u,w}}(s_l).
    \end{equation}
\end{proof}

\newpage
\section{Convergence Analysis} \label{append:convergence_algo}

\subsection{Assumption for Action-value Nondegeneracy} \label{append:convergence_algo_assumption}

\textit{Assumption}
There exists at least one state $s \in \mathcal{S}$ such that $\rho^{\pi^*_{u,w}}(s) > 0$ and the centered action-value vectors $\bigl\{\,Q^{\pi^*_{u,w}}(s,a) - \mathbb{E}_{a' \sim \pi^*_{u,w}(\cdot\mid s)}[\,Q^{\pi^*_{u,w}}(s,a')\,]\;:\;a\in\mathcal{A}\bigr\}$ span $\mathbb{R}^{K+L}$.

This condition fails only in degenerate multi-objective settings when for \textit{every} state $s\in\mathcal{S}$ with $\rho^{\pi^*_{u,w}}(s) > 0$, the set $\bigl\{\,Q^{\pi^*_{u,w}}(s,a) - \mathbb{E}_{a' \sim \pi^*_{u,w}(\cdot\mid s)}[\,Q^{\pi^*_{u,w}}(s,a')\,]: a\in \mathcal{A}\bigr\}$ lies entirely within an affine subspace of dimension less than $K+L$ (e.g., the size of an action set is smaller than the number of objectives).

Then $B^{\pi^*_{u,w}}(s) = \mathbb{E}_{a \sim \pi^*_{u,w}(\cdot|s)} \bigg[ ( Q^{\pi^*_{u,w}}(s,a) - \mathbb{E}_{a' \sim \pi^*_{u,w}(\cdot|s)} [Q^{\pi^*_{u,w}}(s,a')] ) ( Q^{\pi^*_{u,w}}(s,a) - \mathbb{E}_{a' \sim \pi^*_{u,w}(\cdot|s)} [Q^{\pi^*_{u,w}}(s,a')] )^\top    \bigg] \in \mathbb{R}^{(L+K) \times (L+K)}$ is positive definite. This is because (i) $\pi^*_{u,w}(a|s) > 0$ for all $a$ from \eqref{eq:optimal_policy} (which has this favorable property that facilitates analysis), and (ii) for any $y \in \mathbb{R}^{K+L} ~ \text{with} ~ y \neq \mathbf{0}$, $y^\top B^{\pi^*_{u,w}}(s) y = \sum_{a} \pi^*_{u,w}(a|s) \bigg( y^\top ( Q^{\pi^*_{u,w}}(s,a) - \mathbb{E}_{a' \sim \pi^*_{u,w}(\cdot|s)} [Q^{\pi^*_{u,w}}(s,a')] ) \bigg)^2 > 0$ as at least one $a$ should satisfy $y^\top ( Q^{\pi^*_{u,w}}(s,a) - \mathbb{E}_{a' \sim \pi^*_{u,w}(\cdot|s)} [Q^{\pi^*_{u,w}}(s,a')] ) \neq 0$.

By Proposition \ref{thm2_twice_differentiable}, we have the Hessian of $\mathcal{L}(u,w)$ as $H[\mathcal{L}(u,w)] = \frac{1}{\beta} \sum_{l=1}^p [ \mu_0^\top(I_p - \gamma T^{\pi^*_{u,w}})^{-1}]_l B^{\pi^*_{u,w}}(s_l) = \frac{1}{\beta} \sum_s \rho^{\pi^*_{u,w}}(s) B^{\pi^*_{u,w}}(s)$ where $p = |\mathcal{S}|$ and $\rho^{\pi^*_{u,w}}(s) = \sum_{t=0}^\infty \gamma^t \text{Pr}(s_t=s | \pi^*_{u,w}, T, \mu_0)$. Therefore, $H[\mathcal{L}(u,w)]$ is positive definite under the assumption.

\subsection{Existence of $\lambda$} \label{append:existence_lambda}

Let $R = \{ \rho \in \mathbb{R}^{|\mathcal{S}| |\mathcal{A}| }_+ | \sum_{a'} \rho(s',a') = \mu_0(s') + \gamma \sum_{(s,a)} T(s' | s,a) \rho(s,a) ,~\forall s'  \}$ be the set of occupancy measures, which is convex and closed. The set $R$ is also bounded since $\rho(s,a) \leq \frac{1}{1-\gamma} ~ \forall (s,a)$ for $\gamma \in [0,1)$. Under Slater condition with strict feasibility, there exists an occupancy measure $\bar{\rho} \in R$ and positive constants of $\{ \nu_l \}_{l=1}^L$ such that
\begin{equation} 
    \sum_{(s,a)} c^{(l)}(s,a) \bar{\rho}(s,a) = J_{K+l}(\pi^{\bar{\rho}})  \geq C^{(l)} + \nu_l, ~~ l =1, \cdots, L.
\end{equation}
Although $\sum_{(s,a)} c^{(l)}(s,a)\rho(s,a)$ is linear in $\rho$, so that satisfying equality already implies Slater condition, assuming strict feasibility is a mild condition and is standard in prior constrained MDP literature \cite{muller2024truly}.

Define the Lagrangian of the primal problem as:
\begin{equation}
    \tilde{L}(\rho, u) = \min_{1 \leq k \leq K} \bigg(  \sum_{(s,a)} r^{(k)}(s,a) \rho(s,a) \bigg) + \beta \sum_{s} \mathcal{H}_\rho(s) \rho(s) + \sum_{l=1}^L u_l ( \sum_{(s,a)} c^{(l)}(s,a) \rho(s,a) - C^{(l)} )
\end{equation}
with $\rho \in R$ and $u \in \mathbb{R}^L_+$. Let the Lagrangian dual be $D(u) = \max_{\rho \in R} \tilde{L}(\rho, u)$. Let $P^* = \min_{1 \leq k \leq K} \bigg(  \sum_{(s,a)} r^{(k)}(s,a) \rho^*(s,a) \bigg) + \beta \sum_{s} \mathcal{H}_{\rho^*}(s) \rho^*(s)$ denote the optimal value of the primal problem, where $\rho^* \in R$ is an optimal solution that satisfies the $L$ constraints.

By strong duality, we have $P^* = \min_{u \in \mathbb{R}^L_+} D(u) = \tilde{L}(\rho^*, u^*)$ where $u^* \in \mathbb{R}^L_+$ is an optimal dual variable corresponding to the constraints: $\sum_{(s,a)} c^{(l)}(s,a) \rho(s,a) \geq C^{(l)}, ~ l =1, \cdots, L$. Also,
\begin{equation}
    P^* =  \max_{\rho \in R} \tilde{L}(\rho, u^*) \geq \tilde{L}(\bar{\rho}, u^*) \geq \min_{1 \leq k \leq K} \bigg(  \sum_{(s,a)} r^{(k)}(s,a) \bar{\rho}(s,a) \bigg) + \beta \sum_{s} \mathcal{H}_{\bar{\rho}}(s) \bar{\rho}(s) + \sum_{l=1}^L u_l^* \nu_l
\end{equation}
where the second inequality follows from the definition of $\{ \nu_l \}_{l=1}^L$. Thus,
\begin{equation}
    0 \leq u_l^* \leq \frac{P^* - (\min_{1 \leq k \leq K} J_k(\pi^{\bar{\rho}})  + \beta \sum_{s} \mathcal{H}_{\bar{\rho}}(s) \bar{\rho}(s))}{\nu_l} := \bar{u}_l, \forall l.
\end{equation}

Therefore, it suffices to restrict attention to the compact set $u \in \Pi_{l=1}^L [0, \bar{u}_l] =: \mathcal{C}$ instead of $\mathbb{R}^L_+$ when searching for an optimal dual variable corresponding to the constraints $\sum_{(s,a)} c^{(l)}(s,a) \rho(s,a) \geq C^{(l)}, ~ l =1, \cdots, L$.

Next, the Hessian is explicitly written as follows:
\begin{equation}
    H[\mathcal{L}(u,w)] = \frac{1}{\beta} \sum_{l=1}^p [ \mu_0^\top (I - \gamma T^{\pi^*_{u,w}})^{-1}  ]_{l} B^{\pi^*_{u,w}}(s_l).
\end{equation}

From the differentiability of $v^*_{u,w}$ established in Theorem 3.3, $v^*_{u,w}$ is continuous in $(u,w)$. By their definitions, this continuity extends to $Q^*_{u,w}$, $\pi^*_{u,w}$, each entry of $B^{\pi^*_{u,w}}$ and $H[\mathcal{L}(u,w)]$. For two real symmetric matrices  $A$ and $B$, the minimum eigenvalue $\lambda_{\text{min}}(\cdot)$ satisfies:
\begin{equation}
    |\lambda_{\text{min}}(A) - \lambda_{\text{min}}(B)| \leq \| A - B \|_2 \leq \| A - B \|_F,
\end{equation}
where $\| \cdot \|_F$ is the Frobenius norm.

Therefore, $\lambda_{\text{min}}(H[\mathcal{L}(u,w)])$ is continuous in $(u,w)$ over the compact set $\mathcal{C} \times \Delta^K$. By the extreme value theorem, there exists $\lambda = \min_{(u,w) \in \mathcal{C} \times \Delta^K} \lambda_{\text{min}}(H[\mathcal{L}(u,w)]) > 0$ (with $\lambda \leq \alpha$).

\subsection{Proof of Convergence Analysis} \label{append:convergence_algo_proof}

Let $\lambda_{\text{min}}( A )$ be the minimum eigenvalue of a real symmetric matrix $A$. For simplicity, we denote $\lambda := \lambda_{\text{min}}( H[\mathcal{L}(u,w)] )$. Then $0 < \lambda \leq \alpha$ \citep{bubeck2015convex} and $\mathcal{L}(u,w)$ is $\lambda$-strongly convex.

\textbf{Theorem \ref{thm_convergence_algo}}
Let $(u^*,w^*)$ denote an optimal solution to \eqref{eq:equivalent_objective}. For each outer-loop index $m \geq 1$ in Algorithm \ref{alg:constrained_morl}, let $Q^*_{u^m, w^m}$ denote the fixed point of \eqref{eq:sql} with $[u;w] = [u^m;w^m]$, and let $Q^m$ denote the Q-function after completing the $m$-th inner-loop update. For each $m$, assume $\| Q^m - Q^*_{u^m, w^m} \|_\infty < \epsilon$ for some $\epsilon > 0$. Then for $m \geq 1$,
\begin{equation}
    \| [u^m;w^m] - [u^*;w^*] \|_{2} \leq  (1-\frac{\lambda}{\alpha})^m \| [u^0;w^0] - [u^*;w^*] \|_{2} + \frac{\sqrt{|\mathcal{S}|}}{\lambda}  \sqrt{\sum_{i=1}^{K+L} \{ r^{(i)}_{\text{max}} \}^2} \frac{1+\gamma }{(1-\gamma)^2} \epsilon.
\end{equation}

\begin{proof}
    By the definition in \eqref{eq:optimal_policy}, we have the optimal policy $\pi^*_{u^m, w^m}(a | s) = \frac{ \exp ( \frac{1}{\beta} Q^*_{u^m, w^m}(s,a) )  }{\sum_{a'} \exp ( \frac{1}{\beta} Q^*_{u^m, w^m}(s,a') )   }$ when $(u,w) = (u^m,w^m)$. According to Theorem \ref{thm1_differentiable}, we have $\nabla_{(u,w)}\mathcal{L}(u^m,v^m) = [\sum_s \mu_0(s) v_c^{\pi^*_{u^m,v^m}}(s) - [C^{(1)},\cdots,C^{(L)}]^\top ; \sum_s \mu_0(s) v_r^{\pi^*_{u^m,v^m}}(s) ] \in \mathbb{R}^{L+K}$.

    We also have $\tilde{\nabla}_{(u,w)}\mathcal{L}(u^m,v^m) := [\sum_s \mu_0(s) v_c^{\pi^m}(s) - [C^{(1)},\cdots,C^{(L)}]^\top ; \sum_s \mu_0(s) v_r^{\pi^m}(s) ] \in \mathbb{R}^{L+K}$, an estimated gradient of $\nabla_{(u,w)}\mathcal{L}(u^m,w^m)$ using $\pi^m$ where $\pi^m(a | s) = \frac{ \exp ( \frac{1}{\beta} Q^m(s,a) )  }{\sum_{a'} \exp ( \frac{1}{\beta} Q^m(s,a') )   }$.

    Let $e_m := \tilde{\nabla}_{(u,w)}\mathcal{L}(u^m,v^m) - \nabla_{(u,w)}\mathcal{L}(u^m,w^m)$. For each $s$, let $v_{r, i}^{\pi}(s) ~ (1 \leq i \leq K)$ and $v_{c, j}^{\pi}(s) ~ (1 \leq j \leq L)$ denote the elements of the $i$-th dimension of $v_{r}^{\pi}(s) \in  \mathbb{R}^K$ and the $j$-th dimension of $v_{c}^{\pi}(s) \in \mathbb{R}^L$, respectively. Then we have
    \begin{align} \label{eq:error_eq1}
        \| e_m \|^2_2 &= \| [\sum_s \mu_0(s) (v_c^{\pi^m}(s) - v_c^{\pi^*_{u^m,v^m}}(s)) ; \sum_s \mu_0(s) (v_r^{\pi^m}(s) - v_r^{\pi^*_{u^m,v^m}}(s)) ] \|^2_2 \nonumber \\
        &= \sum_{i=1}^K \left(  \sum_s \mu_0(s) (v_{r, i}^{\pi^m}(s) - v_{r, i}^{\pi^*_{u^m, w^m}}(s))  \right)^2 + \sum_{j=1}^L \left(  \sum_s \mu_0(s) (v_{c, j}^{\pi^m}(s) - v_{c, j}^{\pi^*_{u^m, w^m}}(s))  \right)^2 \nonumber \\
        &\leq \| \mu_0 \|^2_2 \sum_s \bigg[  \sum_{i=1}^K  (v_{r, i}^{\pi^m}(s) - v_{r, i}^{\pi^*_{u^m, w^m}}(s))^2   + \sum_{j=1}^L   (v_{c, j}^{\pi^m}(s) - v_{c, j}^{\pi^*_{u^m, w^m}}(s))^2     \bigg] 
    \end{align}
    where $\| \mu_0 \|^2_2 = \sum_s (\mu_0(s))^2$ and the inequality holds by Cauchy-Schwarz.

    Since both $\pi^m$ and $\pi^*_{u^m, w^m}$ use softmax parameterization with $Q^m$ and $Q^*_{u^m, w^m}$, respectively, we have
    \begin{equation} \label{eq:softmax_1}
        \forall s, ~ |v_{r, i}^{\pi^m}(s) - v_{r, i}^{\pi^*_{u^m, w^m}}(s)| \leq \frac{(1+\gamma)r^{(i)}_{\text{max}}}{(1-\gamma)^2} \| Q^m - Q^*_{u^m, w^m} \|_\infty ~ (1 \leq i \leq K)
    \end{equation}
    and
    \begin{equation} \label{eq:softmax_2}
        \forall s, ~ |v_{c, j}^{\pi^m}(s) - v_{c, j}^{\pi^*_{u^m, w^m}}(s)| \leq \frac{(1+\gamma)r^{(K+j)}_{\text{max}}}{(1-\gamma)^2} \| Q^m - Q^*_{u^m, w^m} \|_\infty ~ (1 \leq j \leq L)
    \end{equation}
    according to the property of equation (261) in  \citet{yang24softmaxLipschitz}. Combining \eqref{eq:softmax_1}, \eqref{eq:softmax_2}, and  $\| \mu_0 \|_2 \leq 1$ with \eqref{eq:error_eq1} gives
    \begin{align}  \label{eq:error_eq_final}
        \| e_m \|_2 
        &\leq \sqrt{|\mathcal{S}|} \sqrt{\sum_{i=1}^{K+L} \{ r^{(i)}_{\text{max}} \}^2} \frac{1+\gamma }{(1-\gamma)^2} \| Q^m - Q^*_{u^m, w^m} \|_\infty \nonumber \\
        &< \sqrt{|\mathcal{S}|} \sqrt{\sum_{i=1}^{K+L} \{ r^{(i)}_{\text{max}} \}^2} \frac{1+\gamma }{(1-\gamma)^2} \epsilon.
    \end{align}

    Next, we view the projected gradient descent for each outer loop as a proximal gradient descent. We reformulate the optimization in \eqref{eq:equivalent_objective} of
    \begin{equation}
        \min_{u \in \mathbb{R}_{+}^L, w \in \Delta^K} \mathcal{L}(u,w)
    \end{equation}
    as follows:
    \begin{equation} \label{eq:proximal_conversion}
        \min_{ (u,w) \in \mathbb{R}^{L+K}} \mathcal{L}(u,w) + I_{\mathbb{R}_{+}^L \times \Delta^K}(u,w)
    \end{equation}
    where $I_{\mathbb{R}_{+}^L \times \Delta^K}(u,w)$ is the indicator function with its value 0 if $(u,w) \in \mathbb{R}_{+}^L \times \Delta^K$ and $+\infty$ otherwise. $I_{\mathbb{R}_{+}^L \times \Delta^K}$ is convex because its epigraph $\{(u,w,t_e) | t_e \geq 0, (u,w) \in \mathbb{R}_{+}^L \times \Delta^K \}$ is convex. We note that according to Theorem \ref{thm4_smooth}, the smoothness of $\mathcal{L}(u,w)$ is satisfied on $\mathbb{R}^{L+K}$, which makes \eqref{eq:proximal_conversion} valid. We also note that we computed the smoothness coefficient $\alpha = \frac{1}{\beta (1-\gamma)} \sum_{i=1}^{K+L} \left( \frac{r^{(i)}_{\text{max}}}{1-\gamma} \right)^2$ of $\mathcal{L}$ in Appendix \ref{append:thm_smooth}.
    
    Applying the error bound in \eqref{eq:error_eq_final} to the analysis of inexact proximal gradient method \citep{schmidt11inexactproximalgradient}, we have 
    \begin{align}
        \| [u^m;w^m] - [u^*;w^*] \|_{2} &\leq (1-\frac{\lambda}{\alpha})^m \| [u^0;w^0] - [u^*;w^*] \|_{2} + \frac{1}{\alpha} \sum_{i=1}^m (1-\frac{\lambda}{\alpha})^{m-i} \| e_i \|_2 \nonumber \\
        &\leq  (1-\frac{\lambda}{\alpha})^m \| [u^0;w^0] - [u^*;w^*] \|_{2} + \frac{\sqrt{|\mathcal{S}|}}{\lambda}  \sqrt{\sum_{i=1}^{K+L} \{ r^{(i)}_{\text{max}} \}^2} \frac{1+\gamma }{(1-\gamma)^2} \epsilon.
    \end{align}
  
    This is achieved because we use the convex optimization method from \citet{DBLP:journals/corr/WangC13a} for projection onto the simplex $\Delta^K$, and apply non-negativity clipping for projection onto $\mathbb{R}_+^L$, both of them induce zero error in each phase of proximal objective update as it is fully deterministic and avoids randomized procedures.
    
    It remains to check whether $I_{\mathbb{R}_{+}^L \times \Delta^K}$ in \eqref{eq:proximal_conversion} is a lower semi-continuous proper convex function \citep{schmidt11inexactproximalgradient}. $I_{\mathbb{R}_{+}^L \times \Delta^K}$ is lower semi-continuous because $\mathbb{R}_{+}^L \times \Delta^K$ is closed, and it is also proper convex since $I_{\mathbb{R}_{+}^L \times \Delta^K}$ never attains $-\infty$ and $\mathbb{R}_{+}^L \times \Delta^K$ is non-empty.

\end{proof}

\subsection{Convergence Analysis for Degenerate Case} \label{append:convergence_algo_proof_degenerate}

\begin{theorem}\label{thm_convergence_algo_degenerate}
    Let $(u^*,w^*)$ denote an optimal solution to \eqref{eq:equivalent_objective}. For each outer-loop index $m \geq 1$ in Algorithm \ref{alg:constrained_morl}, let $Q^*_{u^m, w^m}$ denote the fixed point of \eqref{eq:sql} with $[u;w] = [u^m;w^m]$, and let $Q^m$ denote the Q-function after completing the $m$-th inner-loop update. For each $m$, assume $\| Q^m - Q^*_{u^m, w^m} \|_\infty < \epsilon_m$ for some $\epsilon_m > 0$. Then for $m \geq 1$,
    \begin{equation}
        \mathcal{L}( \frac{1}{m} \sum_{i=1}^m (u^i, w^i) ) - \mathcal{L}(u^*,w^*) \leq \frac{\alpha}{2m} ( \| [u^0;w^0] - [u^*;w^*] \|_{2} ~ + \frac{2 M}{\alpha} \sum_{i=1}^m \epsilon_i  )^2
    \end{equation}
    where $M = \sqrt{|\mathcal{S}|} \sqrt{\sum_{j=1}^{K+L} \{ r^{(j)}_{\text{max}} \}^2} \frac{1+\gamma }{(1-\gamma)^2}$.
\end{theorem}

\begin{proof}
Using an analysis of inexact proximal gradient method \citep{schmidt11inexactproximalgradient} using the same logic in the proof of Theorem \ref{thm_convergence_algo} (Appendix \ref{append:convergence_algo_proof}), we have 

\begin{equation}
    \mathcal{L}( \frac{1}{m} \sum_{i=1}^m (u^i, w^i) ) - \mathcal{L}(u^*,w^*) \leq \frac{\alpha}{2m} ( \| [u^0;w^0] - [u^*;w^*] \|_{2} ~ + \frac{2}{\alpha} \sum_{i=1}^m \| e_i \|_2   )^2
\end{equation}
where $e_i := \tilde{\nabla}_{(u,w)}\mathcal{L}(u^i,w^i) - \nabla_{(u,w)}\mathcal{L}(u^i,w^i)$ is the $i$-th gradient error and
\begin{equation}
    \| e_i \|_2 < \sqrt{|\mathcal{S}|} \sqrt{\sum_{j=1}^{K+L} \{ r^{(j)}_{\text{max}} \}^2} \frac{1+\gamma }{(1-\gamma)^2} \epsilon_i = M \epsilon_i
\end{equation}
from \eqref{eq:error_eq_final}.
\end{proof}

We note that the error of $ \mathcal{L}( \frac{1}{m} \sum_{i=1}^m (u^i, w^i) ) - \mathcal{L}(u^*,w^*)$ decreases at rate $O(\frac{1}{m})$ when $\{\epsilon_i\}_{i=1}^\infty$ is summable (e.g., $\epsilon_m = O(\frac{1}{m^{1+\delta}})$ with $\delta > 0$).

\newpage
\section{Experimental Details: Tabular Settings} \label{append:tabular}

When generating structured MOMDPs randomly, we first verify whether the generated instances are feasible. To do this, We first consider the following unregularized convex optimization:
\begin{equation} \label{eq:optimal_obj}
    \max_{\rho \geq 0} \min_{1 \leq k \leq K} \bigg(  \sum_{(s,a)} r^{(k)}(s,a) \rho(s,a) \bigg)
\end{equation}
\begin{equation} \label{eq:optimal_flow}
    \sum_{a'} \rho(s',a') = \mu_0(s') + \gamma \sum_{(s,a)} T(s' | s,a) \rho(s,a) ,~\forall s'
\end{equation}
\begin{equation} \label{eq:optimal_const_c}
    \sum_{(s,a)} c^{(l)}(s,a) \rho(s,a) \geq C^{(l)}, ~~ l =1, \cdots, L
\end{equation}
which is equivalently expressed as the following LP by using additional scalar variable $\tilde{c} \in \mathbb{R}$:
\begin{equation} \label{eq:eqiv_optimal_obj}
    \max_{\rho \geq 0, \tilde{c}} \tilde{c}
\end{equation}
\begin{equation} \label{eq:eqiv_optimal_flow}
    \sum_{a'} \rho(s',a') = \mu_0(s') + \gamma \sum_{(s,a)} T(s' | s,a) \rho(s,a) ,~\forall s'
\end{equation}
\begin{equation} \label{eq:eqiv_optimal_const_r}
    \sum_{(s,a)} r^{(k)}(s,a) \rho(s,a) \geq \tilde{c}, ~~ k =1, \cdots, K,
\end{equation}
\begin{equation} \label{eq:eqiv_optimal_const_c}
    \sum_{(s,a)} c^{(l)}(s,a) \rho(s,a) \geq C^{(l)}, ~~ l =1, \cdots, L.
\end{equation}

We want to generate $\mu_0, T, r$, and $c$ in structured MOMDPs to satisfy feasibility and Slater condition by solving the following LP using the pywraplp function from the OR-Tools library:

\begin{equation} \label{eq:feasible_obj}
    \max_{\rho \geq \epsilon_{\text{low}}} ~ 0
\end{equation}
\begin{equation} \label{eq:feasible_flow}
    \sum_{a'} \rho(s',a') = \mu_0(s') + \gamma \sum_{(s,a)} T(s' | s,a) \rho(s,a) ,~\forall s'
\end{equation}
\begin{equation} \label{eq:feasible_const_r}
    \sum_{(s,a)} r^{(k)}(s,a) \rho(s,a) \geq \tilde{c} + \epsilon_{\text{low}}, ~~ k =1, \cdots, K,
\end{equation}
\begin{equation} \label{eq:feasible_const_c}
    \sum_{(s,a)} c^{(l)}(s,a) \rho(s,a) \geq C^{(l)} + \epsilon_{\text{low}}, ~~ l =1, \cdots, L
\end{equation}

where $\epsilon_{\text{low}}$ is used to guarantee the strict feasibility for Slater condition, and we set $\epsilon_{\text{low}}=10^{-4}$. If the LP solver does not find a feasible solution, we regenerate the constrained MOMDP until a feasible instance is found. Once any feasible solution is found, we solve the LP of \eqref{eq:eqiv_optimal_obj}, \eqref{eq:eqiv_optimal_flow}, \eqref{eq:eqiv_optimal_const_r}, and \eqref{eq:eqiv_optimal_const_c} by using LP solver to acquire the optimal max-min value $\tilde{c}^*$. 

Each algorithm is updated iteratively until the maximum change in the $Q$-function between successive iterations falls below $10^{-4}$. We use the following settings: $\gamma=0.8$, $l_{\text{w}}=0.001$, and $\text{ITER}=3000$. For each algorithm, we independently disable the learning of $u$, of $w$, and of both $(u,w)$, initializing $u$ to the zero vector and $w=[1/K, \cdots, 1/K] \in \Delta^K$. Each algorithm is evaluated using three random seeds for each constrained MOMDP setting, resulting in six runs when averaged across the MOMDP classes. All experiments are conducted on an Intel Core i9-10900X CPU @ 3.70GHz.

\section{Experimental Details: Applications} \label{append:resource_allocation}

\subsection{Implementation of Our Algorithm for Applications} \label{append:realloc_implement}

We now leverage the usage of neural network for our algorithm. If we differentiate the both side of $v^{*}_{u, w}(s) = [\mathcal{T}_{u,w} v^{*}_{u, w}](s)$ w.r.t. $u$ and $w$ for all $s$, then we have the following formula:
\begin{equation} \label{eq:grad_w_v_star}
    \forall s, ~ \nabla_w v^{*}_{u, w}(s) = \sum_a \pi^*_{u,w}(a | s)  \bigg( r(s,a) + \gamma \hspace{-0.3em}\sum_{s'}\hspace{-0.2em}T(s'|s,a) \nabla_w v^{*}_{u, w}(s') \bigg).
\end{equation}

\begin{equation} \label{eq:grad_u_v_star}
    \forall s, ~ \nabla_u v^{*}_{u, w}(s) = \sum_a \pi^*_{u,w}(a | s)  \bigg( c(s,a) + \gamma \hspace{-0.3em}\sum_{s'}\hspace{-0.2em}T(s'|s,a) \nabla_u v^{*}_{u, w}(s') \bigg).
\end{equation}

Here, $\pi^*_{u,w}(a | s)$ is defined as in \eqref{eq:optimal_policy}. To ensure stable gradient estimation in continuous state spaces, we parameterize a gradient network to estimate $\nabla_u v^{*}_{u, w}(s)$ and $\nabla_w v^{*}_{u, w}(s)$. For continuous action environments, we also employ an actor network $\pi_\theta$ and implement Algorithm \ref{alg:ours_realloc}. To further stabilize the estimation of the gradient, we add an additional linear layer after the penultimate layer of the actor network $\pi_\theta$, and use its $(L+K)$-dimensional output as the gradient network $g_\theta(s)$. We use the notation $g_\theta$ to indicate that the actor network and the gradient network share parameters and jointly update their lower-layer weights.

\begin{algorithm}[!ht]  
\caption{Proposed Constrained Max-min Algorithm for Continuous Action} 
\begin{algorithmic}[1] \label{alg:ours_realloc}
    \STATE $\pi_\theta$: actor,  $Q_{\phi}$: critic, $Q_{\overline{\phi}}$: target critic, $g_{\theta}$: gradient network, $g_{\overline{\theta}}$: target gradient network, $\mathcal{D}$: replay buffer, $T_{\text{init}}$: initial iteration number, $\tau$: target update ratio, $U$: main iteration number, $U_{s}$: gradient step for critic update, $l_g$: learning rate of the gradient network, $l_0$: initial learning rate of the weight $(u,w)$, $K$: unconstrained reward dimension, $L$: the number of constraints, $C_{th} \in  \mathbb{R}^L$: threshold vector for the constraints
    \STATE Initialize target critic $\overline{\phi} \leftarrow \phi$, target gradient network $\bar{\theta} \leftarrow \theta$, and weights $u^0 \in \mathbb{R}^L_+$, $w^0 \in \Delta^K$.
    \FOR{$j=0, \cdots, T_{\text{init}}-1$} 
        \STATE Rollout sample from $\pi_\theta$ and save it in $\mathcal{D}$. Sample a batch of data $\mathcal{B} \subset \mathcal{D}$. 
        \STATE $Q_\phi$ $\leftarrow$ \textbf{Critic Update}($Q_\phi$, $Q_{\overline{\phi}}$, $\pi_\theta$, $(u^0, w^0)$, $\mathcal{B}$) (Algorithm \ref{alg:sql_c})
        \STATE Update target critic parameter $\overline{\phi} \leftarrow \tau \phi +(1-\tau)\overline{\phi}$.
        \STATE $\pi_\theta$ $\leftarrow$ \textbf{Actor Update}($Q_\phi$, $\pi_\theta, \mathcal{D}$) (Algorithm \ref{alg:sql_actor_c})
    \ENDFOR
    \FOR{$m = 0, 1, 2, \cdots, U-1$}
        \STATE Rollout sample from $\pi_\theta$ and save $(s,a,r,c,s',\pi_{\theta_{\text{old}}}(a|s))$ in $\mathcal{D}$ where $\pi_{\theta_{\text{old}}}(a|s) = \pi_\theta(a|s)$.
        \STATE Update the gradient network $g_\theta$ as follows: \\
        $\theta \leftarrow \theta - l_g \nabla_{\theta} \mathbb{E}_{(s,a,r,c,s',\pi_{\theta_{\text{old}}}(a|s)) \sim \mathcal{D}}   \left[ \left\| \frac{ \pi_{\theta_{m}}(a|s)   }{ \pi_{\theta_{\text{old}}}(a|s) } ( [c; r] + \gamma  g_{\overline{\theta}}(s')  ) - g_\theta(s) \right\|^2 \right]$ \\
        where $\theta_m$ is a frozen copy of the current parameter $\theta$.
        \STATE Update target gradient network parameter $\overline{\theta} \leftarrow \tau \theta +(1-\tau)\overline{\theta}$.
        \STATE Update $(u,w)=(u^m, w^m)$ using the projected gradient descent:
        \begin{equation*} 
            (u^{m+1}, w^{m+1}) =   \mathcal{P}_{K,L} \left[(u^{m}, w^{m}) - l_m  (\mathbb{E}_{s \sim \mu_0}[g_\theta(s)] - [C_{th}; \mathbf{0}_K])  \right].
        \end{equation*}
        \STATE Schedule current learning rate of the weight $l_m$.
        \FOR{$j=0, \cdots, U_s-1$} 
            \STATE Sample a batch of data $\mathcal{B} \subset \mathcal{D}$. 
            \STATE $Q_\phi$ $\leftarrow$ \textbf{Critic Update}($Q_\phi$, $Q_{\overline{\phi}}$, $\pi_\theta$, $(u^{m+1}, w^{m+1})$, $\mathcal{B}$)
        \ENDFOR
        \STATE Update target critic parameter $\overline{\phi} \leftarrow \tau \phi +(1-\tau)\overline{\phi}$.
        \STATE $\pi_\theta$ $\leftarrow$ \textbf{Actor Update}($Q_\phi$, $\pi_\theta, \mathcal{D}$)
    \ENDFOR
    \STATE Return $\pi_{\theta}$.
\end{algorithmic}
\end{algorithm}

\begin{algorithm}[!ht] 
\caption{Critic Update} \label{alg:sql_c}
\begin{algorithmic}[2]
    \STATE \textbf{Input}: critic $Q_\phi$, target critic $Q_{\overline{\phi}}$, actor $\pi_\theta$,  weight $(u,w)$, sample batch $\mathcal{B}$
    \STATE{Update the critic parameter $\phi$ as follows}:
    \begin{equation}
        \begin{split}
        \phi \leftarrow \phi - l_c \nabla_{\phi} \frac{1}{|\mathcal{B}|} \sum_{(s,a,r,s') \in \mathcal{B}} \bigg( 
        \sum_{k=1}^K w_k r^{(k)}(s,a) + \sum_{l=1}^L u_l c^{(l)}(s,a) \\
        + \gamma \beta \log \mathbb{E}_{a' \sim \pi_{\theta}} \left[ 
        \frac{ \exp \left( Q_{\overline{\phi}}(s',a') / \beta \right) }{\pi_{\theta}(a'|s')} 
        \right] - Q_{\phi}(s,a) \bigg)^2
        \end{split}
    \end{equation}
    where $l_c$ is a critic learning rate.
    \STATE \textbf{Output}: Updated critic $Q_\phi$ 
\end{algorithmic}
\end{algorithm}

\begin{algorithm}[!ht] 
\caption{Actor Update} \label{alg:sql_actor_c}
\begin{algorithmic}[3]
    \STATE \textbf{Input}: critic $Q_\phi$, actor $\pi_\theta$, replay buffer $\mathcal{D}$
    \STATE{Sample a batch of data $\mathcal{B} \subset \mathcal{D}$ and find the actor satisfying the following:}
    \begin{equation}
        \theta \leftarrow \arg \min_{\theta} \mathbb{E}_{s \sim \mathcal{B}} \mathbb{E}_{a \sim \pi_\theta(\cdot | s)}  \left[ \beta \log \pi_\theta(a|s) - Q_{\phi}(s, a)    \right].
    \end{equation} 
    \STATE \textbf{Output}: Updated actor $\pi_\theta$ 
\end{algorithmic}
\end{algorithm}

\subsection{Unconstrained Max-min MORL Algorithm } \label{append:modify_maxmin}

\begin{algorithm}[H]  
\caption{Max-min GS for Continuous Action (Our modification from \citet{park24maxmin}) } 
\begin{algorithmic}[1]
    \STATE $\pi_\theta$: actor,  $Q_{\phi}$: critic, $Q_{\overline{\phi}}$: target critic, $\mathcal{D}$: replay buffer, $T_{\text{init}}$: initial iteration number, $\tau$: target update ratio, $U$: main iteration number, $U_{s}$: gradient step for critic update, $N_{s}$: number of perturbed samples, $\mu$: perturbation parameter, $l_0$: initial learning rate of the weight $w$, $K$: reward dimension
    \STATE Initialize target critic $\overline{\phi} \leftarrow \phi$ and weight $w^0 \in \Delta^K$.

    \FOR{$j=0, \cdots, T_{\text{init}}-1$} 
        \STATE Rollout sample from $\pi_\theta$ and save it in $\mathcal{D}$. Sample a batch of data $\mathcal{B} \subset \mathcal{D}$. 
        \STATE $Q_\phi$ $\leftarrow$ \textbf{Critic Update}($Q_\phi$, $Q_{\overline{\phi}}$, $\pi_\theta$, $w^0$, $\mathcal{B}$) (Algorithm \ref{alg:sql_c} without the term of $\sum_{l=1}^L u_l c^{(l)}(s,a)$)
        \STATE Update target critic parameter $\overline{\phi} \leftarrow \tau \phi +(1-\tau)\overline{\phi}$.
        \STATE $\pi_\theta$ $\leftarrow$ \textbf{Actor Update}($Q_\phi$, $\pi_\theta, \mathcal{D}$) (Algorithm \ref{alg:sql_actor_c})
    \ENDFOR
    \FOR{$m = 0, 1, 2, \cdots, U-1$}
        \STATE Rollout sample from $\pi_\theta$ and save it in $\mathcal{D}$.
        \STATE Generate $N_{s}$ perturbed weights $\{ w^m + \mu u^{m}_n \}_{n=1}^{N_{s}}$, $u^{m}_n \sim \mathcal{N}(0, I_K)$.
        \STATE Make $N_{s}$ copies of $Q_\phi: \{ \hat{Q}_{\phi, \text{copy}, n} \}_{n=1}^{N_{s}}$. Sample a common batch of data $\mathcal{B}_c \subset \mathcal{D}$.
        \FOR{$n=1,\cdots,N_s$}
            \STATE $\hat{Q}_{w^m + \mu u^{m}_n, \text{copy}, n}$ $\leftarrow$ \textbf{Critic Update}($\hat{Q}_{\phi, \text{copy}, n}$, $Q_{\overline{\phi}}$, $\pi_\theta$, $w^m + \mu u^{m}_n$, $\mathcal{B}_c$)
        \ENDFOR
        \STATE Calculate $\hat{L}(w^m + \mu u^{m}_n) = \mathbb{E}_{s \sim \mu_0}  \bigg[ \beta \log  \mathbb{E}_{a \sim \pi_{\theta}} \left[ \frac{ \exp [\hat{Q}_{w^m + \mu u^{m}_n, \text{copy}, n}(s,a) / \beta] }{\pi_{\theta}(a|s)} \right]    \bigg]$. 
        \STATE Conduct linear regression using $\{w^m + \mu u^{m}_n, \hat{L}(w^m + \mu u^{m}_n)  \}_{n=1}^{N_{s}}$ and calculate the linear weight $a_m$. Discard $\{ \hat{Q}_{w^m + \mu u^{m}_n, \text{copy}, n} \}_{n=1}^{N_{s}}$.
        \STATE Update $w=w^m$ using the projected gradient descent:
        \begin{equation*} 
            w^{m+1} =   \text{proj}_{ \Delta^K } \left(w^{m} - l_m  a_m  \right).
        \end{equation*}
        \STATE Schedule current learning rate of the weight $l_m$.
        \FOR{$j=0, \cdots, U_s-1$} 
            \STATE Sample a batch of data $\mathcal{B} \subset \mathcal{D}$. 
            \STATE $Q_\phi$ $\leftarrow$ \textbf{Critic Update}($Q_\phi$, $Q_{\overline{\phi}}$, $\pi_\theta$, $w^{m+1}$, $\mathcal{B}$)
        \ENDFOR
        \STATE Update target critic parameter $\overline{\phi} \leftarrow \tau \phi +(1-\tau)\overline{\phi}$.
        \STATE $\pi_\theta$ $\leftarrow$ \textbf{Actor Update}($Q_\phi$, $\pi_\theta, \mathcal{D}$)
    \ENDFOR
    \STATE Return $\pi_{\theta}$.
\end{algorithmic}
\end{algorithm}

\clearpage
\subsection{{Building Thermal Control}} \label{append:building_overview}
\begin{table}[!ht]
    \centering
    \caption{Hyperparameters for algorithms in building thermal control environment}
    \label{tab:hyperparameters_comparison_build}
    \begin{tabular}{ll}
    \toprule
    \textbf{Parameter} & \textbf{Value} \\
    \midrule
    \multicolumn{2}{l}{\textbf{Shared}} \\
    \midrule
    optimizer & Adam (Kingma \& Ba, 2015) \\
    discount (\(\gamma\)) & 0.99 \\
    target smoothing ratio (\(\tau\)) & 0.001 \\
    reward dimension & 3 \\
    max episode step & 480 \\
    replay buffer size & \(8\times 10^5\) \\
    hidden units per layer & 64 \\
    minibatch size & 32 \\
    activation function & ReLU \\
    entropy coefficient & 0.05 \\
    weight learning rate & \(1 \times 10^{-6}\) \\
    weight scheduling & \(1/\sqrt{t}\) \\
    \midrule
    \multicolumn{2}{l}{\textbf{Constrained Max-min MORL (Ours)}} \\
    \midrule
    constraint type & maximize \\
    constraint dimension & 1 \\
    constraint epsilon & 1.0 \\
    constraint threshold & \(-180\) \\
    main learning rate & \(7.5 \times 10^{-4}\) \\
    gradient steps for critic update & 3 \\
    gradient estimation learning rate & \(5 \times 10^{-5}\)\\
    gradient estimation steps & 1 \\
    gradient target smoothing ratio & 0.001 \\
    \midrule
    \multicolumn{2}{l}{\textbf{Unconstrained Max-min MORL (Gaussian) }} \\ 
    \midrule
    main learning rate & \(7.5 \times 10^{-4}\) \\
    perturbation \(q\) learning rate & 0.073 \\
    perturbation gradient steps & 1 \\
    gradient steps for critic update & 3 \\
    perturbation \(q\)-copies & 10 \\
    perturbation noise std-dev & 0.01 \\
    \midrule
    \multicolumn{2}{l}{\textbf{Unconstrained Max-min MORL (ARAM)}} \\
    \midrule
    main learning rate & \(7.5 \times 10^{-4}\) \\
    CI coefficient $\eta$ & 0.2 \\
    MD coefficient $\lambda$ & 0.03 \\
    \midrule    
    \multicolumn{2}{l}{\textbf{Max-average SAC with a Lagrangian Relaxation}} \\
    \midrule
    constraint type & minimize \\
    initial lambda & 1.0 \\
    main learning rate (actor/critic) & \(3 \times 10^{-4}\) \\
    constraint threshold & 180 \\
    entropy coefficient & 0.05 \\
    lambda learning rate & 0.005 \\
    \midrule
    \multicolumn{2}{l}{\textbf{Unconstrained Max-average SAC}} \\
    \midrule
    main learning rate (actor/critic) & \(3 \times 10^{-4}\) \\
    \bottomrule
    \end{tabular}
\end{table}

{We use a building thermal control environment from SustainGym \citep{yeh2023sustaingym}, which simulates a commercial building with $N_{\text{zone}}=23$ thermal zones. We group these zones into three sections (bottom, middle, and top) and define a comfort objective for each section. Accordingly, in our experiments we set $K=3$ and $L=1$, i.e., three reward objectives and a single cost constraint.}

{At each timestep $t$, the agent observes a normalized continuous state $s_t \in \mathbb{R}^{29}$. The underlying (unnormalized) state is constructed by concatenating: (i) the temperatures of all 23 zones, (ii) outdoor temperature, (iii) global horizontal irradiance (GHI), (iv) ground temperature, (v) occupancy power, (vi) carbon intensity, and (vii) electricity price. The environment clips each component to predefined bounds and linearly rescales it to $[0,1]$ before returning $s_t$.}

{The action is a continuous vector $a_t \in \mathbb{R}^{23}$ with each component bounded in $[-1,1]$, representing zone-level heating/cooling power commands (cooling is negative and heating is positive). The reward is a $K$-dimensional vector $\mathbf{r}_t \in \mathbb{R}^{3}$ that measures comfort for three zone groups. The 23 zones are partitioned into bottom (zones 1--9), middle (zones 10--16), and top (zones 17--23), and each reward entry increases when the temperatures in the corresponding group stay closer to the target temperatures. We define a single scalar energy cost signal $c_t$ from the aggregated magnitude of the zone-level power commands, specifically
$c_t = \|a_t\|_1 / s$, where we tune the scaling factor $s$ (set to $s=5$) to match the reward scale, and use $c_t$ as the cost for the energy constraint ($L=1$).}

{Each episode consists of 480 timesteps. The total training budget is $8\times 10^{5}$ timesteps, with evaluations conducted once every two training episodes. During each evaluation, one episode is run and the cumulative discounted sums of the $(L+K)$-dimensional vector are computed. These experiments were conducted using an NVIDIA RTX 4000 SFF Ada GPU (20GB) across five random seeds. }

\subsection{Effect of $\beta$ and Wall-Clock Time Comparison} \label{append:building_ablation}

\begin{table}[h]
\centering
\caption{Performance comparison with different values of $\beta$ in Building environment ($C_{th}=180$).}
\label{tab:beta_building}
\begin{tabular}{crr}
\toprule
$\beta$ & \begin{tabular}[c]{@{}c@{}}Cost sum\\ ($C_{th} = 180$)\end{tabular} & \begin{tabular}[c]{@{}c@{}}Minimum\\ return ($\uparrow$)\end{tabular} \\
\midrule
0.1 & 152.6 & 625.6 \\
0.05 & 178.7 & 639.8 \\
0.01 & 269.8 & 563.6 \\
\bottomrule
\end{tabular}
\end{table}

\begin{table}[h]
    \centering
    \caption{Wall-clock time comparison per episode in Building environment}
    \label{tab:computation_cost}
    \begin{tabular}{l c}
        \toprule
        Algorithm & Training Time (s) \\
        \midrule
        Random              & $0.2$  \\
        MA-SAC              & $7.3$  \\
        MA-SAC-L            & $13.5$  \\
        Ours                & $18.4$  \\
        Max-min GS  & $48.6$  \\
        ARAM                & $1.7$  \\
        \bottomrule
    \end{tabular}
\end{table}

The entropy term is used not only to encourage exploration but also to resolve the indeterminacy issue in solving the dual of the max-min optimization problem \cite{park24maxmin}. The impact of $\beta$ is seen in Table \ref{tab:beta_building}: The max-min performance in the Building environment is higher for moderate values of $\beta$ (e.g., 0.1, 0.05) than for very small values (e.g., $\beta=0.01$). 

Table~\ref{tab:computation_cost} reports the wall-clock training time per episode for each algorithm. Incorporating constraints increases the training time from MA-SAC to MA-SAC-L. Although our algorithm requires more computation than MA-SAC-L, it achieves a better balance between max-min fairness and constraint satisfaction, as demonstrated in Table~\ref{tab:results_cost_max_queue_length_sum}. Furthermore, our algorithm requires significantly less computation than Max-min GS.

\clearpage
\subsection{Multi-Objective Locomotion Control}
\label{append:hyperparam_ant}
\begin{table}[!ht]
    \centering
    \caption{Hyperparameters for algorithms in multi-objective locomotion control environment}
    \label{tab:hyperparameters_comparison_ant}
    \begin{tabular}{ll}
    \toprule
    \textbf{Parameter} & \textbf{Value} \\
    \midrule
    \multicolumn{2}{l}{\textbf{Shared}} \\
    \midrule
    optimizer & Adam (Kingma \& Ba, 2015) \\
    discount (\(\gamma\)) & 0.99 \\
    target smoothing ratio (\(\tau\)) & 0.001 \\
    reward dimension & 2 \\
    max episode step & 1000 \\
    replay buffer size & \(1\times 10^6\) \\
    hidden units per layer & 64 \\
    minibatch size & 32 \\
    activation function & ReLU \\
    entropy coefficient & 0.05 \\
    weight learning rate & 0.001 \\
    weight scheduling & \(1/\sqrt{t}\) \\
    \midrule
    \multicolumn{2}{l}{\textbf{Constrained Max-min MORL (Ours)}} \\
    \midrule
    constraint type & maximize \\
    constraint dimension & 1 \\
    constraint epsilon & 1.0 \\
    constraint threshold & \(-50\) \\
    main learning rate & \(7.5 \times 10^{-4}\) \\
    gradient steps for critic update & 3 \\
    gradient estimation learning rate & \(2.5 \times 10^{-5}\)\\
    gradient estimation steps & 1 \\
    gradient target smoothing ratio & 0.001 \\
    \midrule
    \multicolumn{2}{l}{\textbf{Unconstrained Max-min MORL (Gaussian) }} \\ 
    \midrule
    main learning rate & \(7.5 \times 10^{-4}\) \\
    perturbation \(q\) learning rate & 0.073 \\
    perturbation gradient steps & 1 \\
    gradient steps for critic update & 3 \\
    perturbation \(q\)-copies & 10 \\
    perturbation noise std-dev & 0.01 \\
    \midrule
    \multicolumn{2}{l}{\textbf{{Unconstrained Max-min MORL (ARAM)}}} \\
    \midrule
    main learning rate & \(7.5 \times 10^{-4}\) \\
    CI coefficient $\eta$ & 0.2 \\
    MD coefficient $\lambda$ & 0.03 \\
    \midrule    
    \multicolumn{2}{l}{\textbf{Max-average SAC with a Lagrangian Relaxation}} \\
    \midrule
    constraint type & minimize \\
    initial lambda & 1.0 \\
    main learning rate (actor/critic) & \(3 \times 10^{-4}\) \\
    constraint threshold & 50 \\
    entropy coefficient & 0.05 \\
    lambda learning rate & 0.001 \\
    \midrule
    \multicolumn{2}{l}{\textbf{Unconstrained Max-average SAC}} \\
    \midrule
    main learning rate (actor/critic) & \(3 \times 10^{-4}\) \\
    \bottomrule
    \end{tabular}
\end{table}

\subsection{Evaluation with Different $C_{th}$}\label{append:ant_different_c}

\begin{table}[h]
    \centering
    \caption{MoAnt-v5 results over five seeds, with the two constraint-satisfying algorithms highlighted in \textbf{bold} ($C_{th}=40$).}
    \label{tab:threshold_40}
    \begin{tabular}{l c c}
        \toprule
        Algorithm &
        \makecell{Cost sum \\ $(C_{th}=40)$} &
        \makecell{Minimum \\ return $(\uparrow)$} \\
        \midrule
        Random     & $146.5$ & $48.2$  \\
        MA-SAC     & $275.3$ & $98.8$  \\
        MA-SAC-L   & $\mathbf{36.8}$  & $78.8$  \\
        Ours      & $\mathbf{36.5}$ & $92.9$ \\
        Max-min GS & $111.7$ & $92.7$  \\
        ARAM       & $620.7$ & $101.3$ \\
        \bottomrule
    \end{tabular}
\end{table}

We set the threshold value $C_{th} = 40$ in the MoAnt environment, which is more stringent than the value $C_{th} = 50$ used in Table \ref{tab:moantv5_results}. Our method consistently balances max-min fairness with constraint satisfaction, meeting the constraint while achieving superior max-min performance compared to MA-SAC-L.

\clearpage
\subsection{{Greenhouse-gas-emission-aware Traffic
Management}}\label{append:traffic}

\label{append:hyperparam_traffic}
\begin{table}[!ht]
    \centering
    \caption{Hyperparameters for traffic signal control environment}
    \label{tab:hyperparameters_comparison_traffic}
    \begin{tabular}{ll}
    \toprule
    \textbf{Parameter} & \textbf{Value} \\
    \midrule
    \multicolumn{2}{l}{\textbf{Shared}} \\
    \midrule
    optimizer & Adam (Kingma \& Ba, 2015) \\
    discount (\(\gamma\)) & 0.99 \\
    target smoothing ratio (\(\tau\)) & 0.001 \\
    reward dimension & 16 \\
    total seconds per episode & 9000 \\
    delta time (seconds) & 30 \\
    total timesteps & \(1\times 10^5\) \\
    replay buffer size & \(1\times 10^5\) \\
    hidden units per layer & 64 \\
    minibatch size & 32 \\
    activation function & ReLU \\
    entropy coefficient & 0.05 \\
    weight scheduling & \(1/\sqrt{t}\) \\
    \midrule
    \multicolumn{2}{l}{\textbf{Constrained Max-min MORL (Ours)}} \\
    \midrule
    constraint type & minimize \\
    constraint dimension & 1 \\
    constraint epsilon & -1.0 \\
    constraint threshold & \(7.0 \times 10^{4}\)\\
    main learning rate & \(7.5 \times 10^{-4}\) \\
    weight learning rate & 0.01 \\
    gradient steps for critic update & 3 \\
    gradient estimation learning rate & \(1.0 \times 10^{-5}\)\\
    gradient estimation steps & 1 \\
    gradient target smoothing ratio & 0.001 \\
    \midrule
    \multicolumn{2}{l}{\textbf{Unconstrained Max-min MORL (Gaussian)}} \\ 
    \midrule
    main learning rate & \(7.5 \times 10^{-4}\) \\
    weight learning rate & 0.01 \\
    perturbation \(q\) learning rate & 0.073 \\
    perturbation gradient steps & 1 \\
    gradient steps for critic update & 3 \\
    perturbation \(q\)-copies & 20 \\
    perturbation noise std-dev & 0.01 \\
    \midrule
    \multicolumn{2}{l}{\textbf{Unconstrained Max-min MORL (ARAM)}} \\
    \midrule
    main learning rate & 0.001 \\
    CI coefficient $\eta$ & 0.00202 \\
    MD coefficient $\lambda$ & 0.2 \\
    \midrule
    \multicolumn{2}{l}{\textbf{Constrained Max-average PGO}} \\
    \midrule
    main learning rate & 0.01 \\
    constraint type & minimize \\
    constraint threshold & \(7.0 \times 10^{4}\) \\
    constraint learning rate & 0.001 \\
    \midrule
    \multicolumn{2}{l}{\textbf{Unconstrained Max-average PGO}} \\
    \midrule
    main learning rate & 0.001 \\
    \bottomrule
    \end{tabular}
\end{table}

Here, $|\mathcal{A}|=4$. MA-PGO is implemented using the PPO algorithm \citep{schulman17ppo} to maximize the average reward over $K=16$ objectives. Based on MA-PGO, MA-CPGO applies clipping only to the unconstrained reward part (i.e., $w \cdot r$), while leaving the constrained reward part unclipped (i.e., $u \cdot c$) and applies the Lagrangian update, following \citet{liu2019ipointeriorpointpolicyoptimization} to improve constraint satisfaction. We found that applying clipping to $u \cdot c + w \cdot r$ does not guarantee constraint satisfaction. Each method is run for 100k timesteps per seed using five random seeds.

\section{Limitation, Future Work, and Discussion} \label{append:limitations}

In this section, we discuss several limitations of our work and related future research avenues, although our method offers a promising direction for developing constrained MORL algorithms.

First, there is a lack of well-established benchmarks for MORL compared to standard RL settings \citep{hayes22survey}, and even fewer environments are specifically designed for constrained MORL. Additionally, most existing MORL environments have low-dimensional reward spaces (typically fewer than four dimensions) \citep{park2025reward}, which limits the ability to evaluate our algorithm in high-dimensional settings. Developing practical benchmarks for both MORL and constrained MORL is therefore a critical research direction for the community.

Second, while it is common in the constrained MDP literature to assume that feasibility is ensured by appropriately chosen thresholds \citep{tessler2018reward,ha2020learning}, determining such thresholds, that is, setting the constraint set $\{ C^{(l)} \}_{l=1}^L$, is non-trivial in practice outside of simple or tabular domains. Unlike trial-and-error reward design, constraint threshold design is often infeasible or unsafe due to the potential risks and costs involved. Leveraging external sources of information, such as human demonstrations or natural language descriptions, offers a promising path for setting constraint thresholds in constrained RL and MORL. Another possible approach is to infer the constraint values from expert demonstrations, commonly referred to as inverse constrained RL \citep{pmlr-v139-malik21a,subramanian2024confidence}.

Third, while our setting clearly distinguishes rewards from costs, this distinction may be ambiguous in other domains. Determining which objectives should be treated as constraints versus unconstrained rewards can be challenging. As with constraint threshold design, incorporating external guidance could help better structure constrained MORL problems.

Fourth, several constrained RL studies have explored more conservative formulations than those based on expected cumulative cost, for example, using outage probability or quantile-based constraints to manage rare but critical failures in domains such as finance or insurance \citep{yang21wcsac,jung22quantile}. While our current framework and analysis rely on expected cumulative cost, extending it to support such conservative constraint formulations presents a valuable direction for future work.

Lastly, although we assume the convergence of the (action) value function for each weight pair $(u,w)$, it is well known that the combination of function approximation, bootstrapped updates, and off-policy learning can lead to instability and even divergence during training \citep{sutton2018reinforcement, che24deadlytrial}. A theoretical investigation into this so-called \textit{deadly triad}, along with additional convergence guarantees, would further improve the robustness of our algorithm and broaden its applicability to other domains.

Scalarization-based methods are highly valuable, especially because of their interpretability and flexibility in expressing designer preferences. In particular, when incorporating constraints, these methods also make it straightforward to assess constraint satisfaction through the corresponding dual variables. However, linear scalarization cannot recover nonlinear or concave regions of the Pareto frontier, potentially missing desirable trade-off solutions \citep{roijers13survey,hayes22survey}. While mixtures of convex scalarization functions can help approximate concave regions, this often requires careful tuning and may increase computational effort. Addressing the limitations of scalarization-based approaches is indeed valuable.

Finally, we note that the purpose of our theoretical analysis is not to imply that all assumptions will be checked analytically in practice, but rather to provide predictable behavior and guidance for practical usage of our algorithm.


\end{document}